\documentclass{article}

\PassOptionsToPackage{numbers,sort&compress}{natbib}
\usepackage[preprint]{neurips_2026}

\usepackage[utf8]{inputenc}
\usepackage[T1]{fontenc}
\usepackage{hyperref}
\usepackage{url}
\usepackage{booktabs}
\usepackage{amsfonts}
\usepackage{nicefrac}
\usepackage{microtype}
\usepackage{xcolor}

\usepackage{amsmath}
\usepackage{amssymb}
\usepackage{mathtools}
\usepackage{amsthm}
\usepackage{graphicx}
\usepackage{subcaption}
\usepackage{enumitem}
\usepackage[capitalize,noabbrev]{cleveref}
\usepackage{algorithm}
\usepackage{algorithmic}
\usepackage[most]{tcolorbox}
\usepackage{listings}

\lstdefinestyle{mosaicpy}{%
  basicstyle=\ttfamily\scriptsize, language=Python,
  keywordstyle=\color{black!75}\bfseries, commentstyle=\color{black!55}\itshape,
  stringstyle=\color{black!70}, showstringspaces=false,
  breaklines=true, breakatwhitespace=true,
  columns=flexible, keepspaces=true, xleftmargin=0.5em,
  aboveskip=2pt, belowskip=2pt
}

\definecolor{mosaicCreative}{HTML}{6E3FA0}
\definecolor{mosaicCode}{HTML}{1F5FA8}
\definecolor{mosaicReasoning}{HTML}{0F7A6B}
\definecolor{mosaicAdvisory}{HTML}{B25E1C}

\newtcolorbox{mosaicbox}[2][]{%
  sharp corners, boxrule=0.5pt,
  colback=#2!4, colframe=#2!75!black,
  fonttitle=\bfseries\small,
  coltitle=white, colbacktitle=#2!85!black,
  top=1.5mm, bottom=1.5mm, left=2mm, right=2mm,
  before skip=4pt, after skip=6pt,
  fontupper=\small, #1
}

\theoremstyle{plain}
\newtheorem{theorem}{Theorem}[section]
\newtheorem{proposition}[theorem]{Proposition}
\newtheorem{lemma}[theorem]{Lemma}

\theoremstyle{definition}

\newtheorem{assumption}[theorem]{Assumption}
\theoremstyle{remark}
\newtheorem{remark}[theorem]{Remark}

\title{Escaping the Mode Lottery: Multi-Response Training Improves Language Model Generalization}

\author{%
  Hasan Amin \\
  Department of Computer Science\\
  Purdue University\\
  \texttt{hasanamin@purdue.edu} \\
  \And
  Kian Ahrabian \\
  Information Sciences Institute \\
  University of Southern California\\
  \texttt{ahrabian@usc.edu} \\
  \AND
  Ming Yin \\
  Department of Computer Science\\
  Purdue University\\
  \texttt{mingyin@purdue.edu} \\
  \And
  Rajiv Khanna \\
  Department of Computer Science\\
  Purdue University\\
  \texttt{rajivak@purdue.edu} \\
}

\begin{document}

\maketitle

\begin{abstract}
Modern language-model fine-tuning typically pairs each prompt with a single response, even though many prompts admit multiple valid completions. This effectively reduces a multi-modal conditional distribution to a one-sample view, a phenomenon we call the ``mode lottery,'' where training emphasizes a subset of plausible modes while leaving others statistically underrepresented. We study multi-response training (MRT), which retains multiple responses per prompt, and develop a principled account of when and why it helps. Our key insight is that prompts and responses are distinct statistical resources: additional prompts reduce uncertainty about the input distribution, while additional responses reduce uncertainty about the conditional output distribution. This yields a simple variance--budget tradeoff that predicts when retaining multiple responses is worthwhile, shows diminishing returns as prompt-level uncertainty dominates, and explains why large redundant corpora can exhibit an implicit multi-response effect. We further analyze response selection, and show that Random-\(K\)-of-\(N\) is the unbiased default for distributional fine-tuning, reward-based selection can induce persistent mode collapse, and a submodular quality--diversity objective provides an efficient alternative with theoretical guarantees. Controlled simulations validate the predicted variance and selection effects, including a striking failure mode where reward-only selection produces gradients misaligned with the true objective. Across structured and real-world datasets, including a new multi-prompt, multi-response benchmark, MRT consistently improves distributional generalization, with the largest gains in high response-diversity, low prompt-redundancy regimes. Overall, MRT reframes response multiplicity as a data-allocation problem with clear guidance: when responses are cheap and diverse, keeping more than one is not a heuristic, but a statistically grounded choice.\footnote{Code and data available at \url{https://github.com/shasanamin/multi-response-training}.}
\end{abstract}
\section{Introduction}
\label{sec:intro}

Modern language model training is largely organized around a quiet convention of accompanying one prompt with one response.
This convention is natural for classification-style supervision, but it is a poor fit for open-ended generation.
Prompts such as ``\textit{write a function that sorts a list}" admit many correct implementations, and ``\textit{summarize this article}" allow multiple valid framings.
Training on only one response therefore exposes the model to a narrow slice of a multi-modal conditional distribution.
We call this the \emph{mode lottery}: the training set may repeatedly show one plausible mode---often the most rewarded, most common, or easiest-to-generate one---while leaving other valid and valuable modes statistically invisible.

This paper studies \emph{multi-response training} (MRT), i.e., training on \(K>1\) responses per prompt, with standard single-response training (SRT) recovered as the special case \(K=1\).
MRT is already present as a successful heuristic in several modern post-training pipelines, where multiple responses per prompt are generated, filtered, ranked, or reused during supervised and reinforcement-learning stages~\citep{touvron2023llama,dubey2024llama,bai2022constitutional,gulcehre2023reinforced,dong2023raft,yuan2023rrhf}.
Yet the field lacks a clear account of when retaining multiple responses should help, how large the gain can be, and which responses should be kept.
Our goal is to turn this heuristic into a data-allocation principle.

The central observation is that prompts and responses are different statistical resources.
Additional prompts reduce uncertainty about the marginal prompt distribution, while additional responses to the same prompt reduce uncertainty about the conditional response distribution.
These two uncertainties have different costs and different limits.
If prompt-level variation dominates, collecting new prompts is more valuable.
If within-prompt response variation is large and responses are cheap, MRT can substantially improve distributional generalization.
This also explains why very large, redundant corpora can sometimes work well with one response per prompt, as near-duplicate or task-equivalent prompts can provide what we call an \textit{implicit MRT} effect.

We answer three practical questions.
First, \emph{when} should one keep more than one response?
A variance-budget analysis separates prompt-level uncertainty from response-level uncertainty and yields a closed-form rule for the optimal response multiplicity \(K^\star\), depending only on the ratio of within-prompt to across-prompt variance and the relative cost of prompts versus responses.
Second, \emph{by how much} can MRT help?
We give complementary Bernstein and hierarchical Rademacher generalization bounds, with the former giving a variance-sensitive view of diminishing returns and the latter handling the non-i.i.d. structure of repeated prompts plus showing that the response-level complexity can scale with the total number of prompt-response pairs rather than only with responses per prompt.
Third, \emph{which} responses should be kept?
We distinguish distributional, reward-tilted, and quality-coverage objectives. Random-\(K\)-of-\(N\) (RKoN) is the unbiased default for distributional fine-tuning---which is the focus of this work---while reward-only best-of-\(N\) can induce persistent mode collapse and proposed GRADES provides an efficient submodular recipe for retaining a balance of high-quality and diverse responses.

We summarize our contributions below:
\begin{enumerate}[leftmargin=1.25em,itemsep=0pt]
    \item \textbf{A variance-budget principle for response multiplicity.}
    We show that MRT separates prompt-level uncertainty from response-level uncertainty, yielding a simple allocation law for \(K^\star\) and a direct test for when keeping multiple responses is worthwhile.

    \item \textbf{Generalization guarantees for hierarchical response data.}
    We give complementary Bernstein and hierarchical Rademacher bounds. The Bernstein view gives a variance-sensitive dependence on \(K\), while the Rademacher view shows that, despite repeated prompts and non-i.i.d.\ structure, the response-level complexity can scale with the full prompt-response sample size when conditional symmetrization is handled correctly.

    \item \textbf{A target-aware theory of response selection.}
    We extend our analysis from distributional to reward-tilted and quality-coverage targets, establishing how RKoN is unbiased for distributional fine-tuning, reward-only best-of-\(N\) can collapse valid modes, and GRADES gives a greedy quality-diversity selector with a \((1-1/e)\) submodular approximation guarantee.

    \item \textbf{Controlled and large-scale validation.}
    Exact-ground-truth simulations validate the variance law, budget optimum, and selector predictions. LLM fine-tuning on multiple datasets across domains shows the same structure in practical training.
\end{enumerate}

The resulting message is simple yet not unconditional.
MRT is not a claim that every dataset needs many responses per prompt, nor that reward-selected responses are always harmful.
Rather, it provides a principled way to decide when response multiplicity is worth the cost and how to use it.
When prompts are expensive, responses are cheap, and valid completions are diverse, single-response fine-tuning leaves useful supervision on the table.
MRT escapes the mode lottery by training on more of the conditional diversity that frontier language models can already generate.

\noindent \textbf{Related Work.} While multiple candidate completions are commonly used in both training (\textit{e.g.}, rejection sampling)~\citep{bai2022constitutional,touvron2023llama,dubey2024llama,gulcehre2023reinforced,dong2023raft,yuan2023rrhf,zelikman2022star} and inference (\textit{e.g.}, Pass@$K$)~\citep{chen2021evaluating,wang2023selfconsistency,lightman2023lets,snell2024scaling,brown2024monkeys,wu2024inference} pipelines, MRT specifically determines the number and selection criteria for retaining these samples.
Moreover, existing approaches for RLHF-style objectives (\textit{e.g.}, best-of-\(N\) selection)~\citep{beirami2024theoretical,gao2023scaling,coste2023reward,eisenstein2024helping,casper2023open} often change the training distribution by reward maximization.
Our analysis targets non-dominant valid modes in supervised fine-tuning reward-tail dominance conditions, where the target is not simply the highest-reward completion, using a strictly proper scoring rule~\citep{gneiting2007strictly}.
In parallel, we complement approaches to instruction-tuning data preparations~\citep{wei2022finetuned,wang2023selfinstruct,zhou2024lima,xu2024magpie,lu2024instag,liu2024deita,song2024scaling,sun2024principle} by introducing a variance-budget law determining the number of responses that should be kept.
Finally, going beyond the i.i.d. assumption in classical generalization analyses~\citep{shalevshwartz2014understanding,mohri2018foundations} and related-tasks shared structure studies in multi-task and hierarchical Rademacher analyses~\citep{baxter2000model,maurer2016benefit,tripuraneni2021provable,du2021fewshot}, 
We adapt variance-sensitive and conditional Rademacher for hierarchies where responses are conditionally sampled within each cross-task prompt.
Specifically, to balance quality with coverage over valid response modes, GRADES builds on submodular facility-location and diversity-aware subset selection methods~\citep{nemhauser1978analysis,minoux1978accelerated,lin2011class,kirchhoff2014submodularity,wei2015submodularity,sener2018coreset,xia2024less,agarwal2024delift}.
For more details, see \autoref{sec:app_related}.
\section{Multi-Response Training}
\label{sec:mrt}

We formalize multi-response training (MRT) as a minimal extension of standard supervised learning or fine-tuning.
Let $\pi_\theta(\mathbf y \mid \mathbf x)$ be an autoregressive model with parameters $\theta$, trained with token-level cross-entropy loss $\ell(\theta; \mathbf x, \mathbf y)$.
The population objective is $\mathcal L(\theta) = \mathbb{E}_{P(\mathbf x, \mathbf y)}[\ell]$.

Standard single-response training (SRT) observes one response per prompt and minimizes
\(
\hat{\mathcal L}_{\text{SRT}} = \frac{1}{N_p} \sum_{i=1}^{N_p} \ell(\theta; \mathbf x_i, \mathbf y_i).
\)
In contrast, MRT retains $K > 1$ responses per prompt and minimizes
\(
\hat{\mathcal L}_{\text{MRT}} = \frac{1}{N_p} \sum_{i=1}^{N_p} \frac{1}{K} \sum_{j=1}^K \ell(\theta; \mathbf x_i, \mathbf y_{ij}).
\)
SRT is thus the special case $K=1$, while MRT explicitly trains on multiple samples from the conditional response distribution.

\noindent \textbf{Overgeneration and selection.}
In modern post-training pipelines, it is common to generate multiple candidate responses per prompt using a base model or policy.
Formally, for each prompt $\mathbf x$, we obtain a candidate set $\mathcal Y'(\mathbf x) = \{\mathbf y_1, \dots, \mathbf y_n\}$, from which a subset $\mathcal S(\mathbf x) \subseteq \mathcal Y'(\mathbf x)$ of size $K$ is selected for training.
MRT changes the downstream use of this candidate pool: instead of selecting a single response, it trains on multiple retained responses.

This also raises an important design question of which responses should be kept?
While we focus on \textit{Random-$K$-of-$N$ (RKoN)}, which samples $K$ responses uniformly 
at random 
from the candidate pool of $N$ responses \textit{for each prompt}, and serves as the unbiased baseline, we also consider two other prominent settings:
\emph{Reward-based selection} (e.g., best-of-$N$ or top-$K$), which prioritizes high-scoring responses but may concentrate on a narrow subset of modes; and
\emph{diversity-aware selection}, which aims to retain responses that are both high-quality and distinct, capturing broader coverage of valid outputs.

\noindent \textbf{Statistical structure.}
MRT induces a hierarchical data structure.
Prompts $\mathbf x_1, \dots, \mathbf x_N$ are sampled i.i.d.\ from $P(\mathbf x)$, while responses $\mathbf y_{ij}$ are sampled conditionally from $P(\mathbf y \mid \mathbf x_i)$.
This distinction is critical since additional prompts and additional responses play fundamentally different statistical roles.
Increasing $N_p$ reduces uncertainty about the prompt distribution, while increasing $K$ reduces uncertainty about the conditional response distribution.
The benefits and limitations of MRT arise from this asymmetry.

\noindent \textbf{Scope and evaluation.}
Our focus is on distributional fine-tuning, where the objective is to model the conditional distribution $P(\mathbf y \mid \mathbf x)$ rather than to optimize a reward concentrated on a single preferred output.
This setting naturally arises in multi-solution tasks, robustness evaluation, and pass@$K$-style metrics.
Under this perspective, retaining multiple valid responses is not redundant supervision, but additional signal about the conditional structure.

\section{When Does MRT Help? A Variance--Budget Law}
\label{sec:when_mrt}

MRT is not universally beneficial.
Keeping more responses per prompt improves some aspects of learning while leaving others unchanged.
The key question is therefore not whether \(K>1\) can help, but \textit{when response multiplicity is the right use of data budget}.

\noindent \textbf{A structural decomposition.}
The answer follows from a single structural observation that prompts and responses contribute different kinds of information.
New prompts expand coverage of the input distribution, while additional responses refine the model’s estimate of the conditional output distribution.
This asymmetry induces a two-term decomposition.

\begin{lemma}[Prompt--response variance decomposition]
\label{lemma:variance_decomposition}
For any fixed \(\theta\), the variance of multi-response empirical loss decomposes exactly as,
\[
    \operatorname{Var}\!\left[\widehat{\mathcal L}_{\mathrm{MRT}}(\theta)\right]
    =
    \frac{V_{\mathbf x}(\theta)}{N_p}
    +
    \frac{V_{\mathbf y\mid \mathbf x}(\theta)}{N_pK},
\]
where
\(
    V_{\mathbf x}
    \coloneqq
    \operatorname{Var}_{\mathbf x}
    \!\left[
        \mathbb E_{\mathbf y\mid \mathbf x}
        \ell
    \right]
\)
is the across-prompt variance and
\(
    V_{\mathbf y\mid \mathbf x}
    \coloneqq
    \mathbb E_{\mathbf x}
    \!\left[
        \operatorname{Var}_{\mathbf y\mid \mathbf x}
        \ell
    \right]
\)
is the expected within-prompt variance.
\end{lemma}

\noindent
This decomposition makes the role of MRT explicit.
The first term is a \textit{prompt floor} that depends only on the number of distinct prompts \(N_p\) and cannot be reduced by additional responses.
The second term is the \textit{response variance}, which shrinks linearly with number of responses per prompt \(K\).
\textit{MRT helps to the extent that within-prompt response diversity is large relative to across-prompt variation.}
If most uncertainty comes from variation across prompts, collecting more responses is ineffective.
If uncertainty is dominated by variation within prompts, increasing \(K\) is highly beneficial.

\noindent \textbf{From variance to allocation.}
If all examples had equal cost, the optimal strategy would be trivial: maximize the number of independent prompt-response pairs.
MRT becomes nontrivial because prompts and responses have very different acquisition costs.
In modern pipelines, prompts are often human-authored, curated, filtered, or licensed, while additional responses can be generated cheaply from an existing model.
At scale, this asymmetry can be extreme: the marginal cost of new prompts can dominate, while the cost of additional responses is negligible.
Let \(C_p\) denote the cost of obtaining a prompt and \(C_r\) the cost of obtaining an additional response.
Under a fixed budget \(\mathcal B = N_p(C_p + K C_r)\), we obtain:

\begin{proposition}[Cost-aware optimal multiplicity]
\label{prop:budget_allocation}
The variance-minimizing response multiplicity is
\[
    K^\star
    =
    \sqrt{
        \frac{V_{\mathbf y\mid \mathbf x}}{V_{\mathbf x}}
        \cdot
        \frac{C_p}{C_r}
    }.
\]
Moreover, MRT is preferred over single-response training when
\(
    \frac{V_{\mathbf y\mid \mathbf x}}{V_{\mathbf x}}
    >
    \frac{C_r}{C_p}.
\)
\end{proposition}
The optimal \(K^\star\) thus depends on two ratios:
(i) a \textit{statistical ratio} \(V_{\mathbf y\mid \mathbf x}/V_{\mathbf x}\), capturing how diverse responses are relative to prompts, and
(ii) a \textit{cost ratio} \(C_p/C_r\), capturing how expensive prompts are relative to responses.
The square-root form is crucial, and it implies diminishing returns: increasing \(K\) yields rapid gains initially but quickly saturates as the prompt floor dominates.
\textit{MRT is most valuable when responses are cheap and diverse, and prompts are expensive and non-redundant.}

\noindent \textbf{Practical estimation.}
The quantities in Proposition~\ref{prop:budget_allocation} can be estimated from a small validation pool.
Given prompts with multiple sampled responses, one computes variance across prompt means and variance within prompts.
In practice, this requires only tens of prompts and a modest number of responses per prompt, and serves as a reliable diagnostic of whether MRT is likely to help.

\noindent \textbf{Implicit MRT through redundancy.}
The same analysis explains an important empirical phenomenon.
Large corpora can exhibit substantial prompt redundancy, as many practical prompts are paraphrases or near-equivalent tasks (e.g., ``\textit{how hot is the weather today?}'' vs. ``\textit{does it feel really warm today?}'').
In such settings, single-response training can still observe multiple valid responses for the same underlying task, spread across different prompts.
This can induce an \textit{implicit MRT} effect, where prompt redundancy partially substitutes for explicit response multiplicity.
\textit{Explicit MRT is most beneficial in low-redundancy regimes, where each prompt covers a distinct region of the task distribution.}
In contrast, highly redundant corpora already approximate this effect and therefore see smaller gains from increasing \(K\).
We also design experiments to probe this distinction directly.

\section{How Much Can MRT Help? Two Generalization Views}
\label{sec:how_much}

The variance--budget law identifies \emph{when} MRT is useful.
We now quantify \emph{how much} it can help.
The answer is governed by the same structure as before: prompt-level uncertainty and response-level uncertainty behave differently.
The key question is how this structure manifests in generalization.

We present two complementary views.
A Bernstein view is variance-adaptive and gives the sharpest fixed-model characterization.
A hierarchical Rademacher view is uniform and shows that the same gains persist under model selection despite the non-i.i.d.\ structure of repeated prompts.

\noindent \textbf{Variance-adaptive generalization.}
For a fixed model, MRT averages \(N_p\) independent prompt-level quantities, each itself an average over \(K\) responses.
Thus the effective variance is exactly the prompt--response decomposition from Lemma~\ref{lemma:variance_decomposition}.

\begin{theorem}[Bernstein generalization bound]
\label{theorem:bernstein_bound}
Assume \(0 \le \ell(\theta;\mathbf x,\mathbf y) \le B\).
For any fixed \(\theta\), with probability at least \(1-\delta\),
\[
    \left|
    \mathcal L(\theta)
    -
    \widehat{\mathcal L}_{\mathrm{MRT}}(\theta)
    \right|
    \le
    \sqrt{
        \frac{
            2\left(
                V_{\mathbf x}(\theta)
                +
                \frac{V_{\mathbf y\mid \mathbf x}(\theta)}{K}
            \right)
            \log(2/\delta)
        }{N_p}
    }
    +
    \frac{2B\log(2/\delta)}{3N_p}.
\]
\end{theorem}

\noindent
This bound gives a clean quantitative picture.
Increasing \(K\) reduces only the within-prompt term \(V_{\mathbf y\mid \mathbf x}/K\), while the prompt-level term \(V_{\mathbf x}\) remains unchanged.
MRT converts within-prompt diversity into reducible error but cannot reduce the prompt floor.
This helps explain our empirical observation of gains from increasing \(K\) being large initially but saturating as the prompt-level uncertainty dominates.

\noindent \textbf{Uniform generalization under hierarchy.}
The Bernstein bound is tight but require fixed \(\theta\).
Yet training selects \(\theta\) from data, requiring a uniform guarantee.
We need to carefully handle the MRT data as it involves subtle yet significant hierarchy: there are \(N_pK\) samples but only \(N_p\) independent units.
\textit{Two naive analyses fail in opposite ways}:
treating all \(N_pK\) pairs as i.i.d.\ ignores the prompt bottleneck, while analyzing each prompt separately loses the dependence on \(N_p\) and yields only \(O(1/\sqrt{K})\) scaling.
The correct analysis must respect both levels simultaneously.

Let
\(
\mathcal F
=
\bigl\{
    (\mathbf x,\mathbf y)\mapsto \ell(\theta;\mathbf x,\mathbf y)
    :
    \theta\in\Theta
\bigr\},
\)
and
\(
\mathcal G
=
\bigl\{
    \mathbf x \mapsto \mathbb E_{\mathbf y\mid\mathbf x}\ell(\theta;\mathbf x,\mathbf y)
    :
    \theta\in\Theta
\bigr\}.
\)
For a fixed prompt set \(X\), define the \textit{conditional} response-level Rademacher complexity
\[
\mathfrak R_{N_pK}(\mathcal F\mid X)
=
\mathbb E_{\sigma,Y\mid X}
\left[
    \sup_{\theta}
    \frac{1}{N_pK}
    \sum_{i=1}^{N_p}
    \sum_{j=1}^{K}
    \sigma_{ij}\,
    \ell(\theta;\mathbf x_i,\mathbf y_{ij})
\right].
\]

\begin{theorem}[Hierarchical uniform generalization]
\label{theorem:rademacher_bound}
Assume \(0 \le \ell \le B\).
With probability at least \(1-\delta\), uniformly for all \(\theta\in\Theta\),
\[
\mathcal L(\theta)
-
\widehat{\mathcal L}_{\mathrm{MRT}}(\theta)
\le\;
2\,\mathfrak R_{N_p}(\mathcal G)
+
2\,\mathbb E_X[\mathfrak R_{N_pK}(\mathcal F\mid X)] +
3B\sqrt{\frac{\log(4/\delta)}{2N_p}}
+
B\sqrt{\frac{\log(4/\delta)}{2N_pK}}.
\]
\end{theorem}

\noindent
The structure mirrors the Bernstein bound.
Prompt-level terms depend only on \(N_p\), while response-level terms benefit from the full \(N_pK\) sample size.
Furthermore, under standard capacity control
\[
\mathbb E_X[\mathfrak R_{N_pK}(\mathcal F\mid X)]
= O\!\left(\frac{1}{\sqrt{N_pK}}\right),
\]
so increasing \(K\) genuinely improves the uniform bound.

The key technical step is keeping \(\sup_\theta\) outside the per-prompt average so that symmetrization acts jointly over all \(N_pK\) responses.
This preserves the \(1/\sqrt{N_pK}\) rate that naive per-prompt analyses lose.
\textit{Bernstein explains the gain; hierarchical Rademacher certifies that the gain survives model selection.}

\noindent \textbf{Putting the views together.}
Both bounds expose the same mechanism.
MRT reduces only the response-level component of generalization error, and all improvements are mediated through this term.
The prompt-level floor remains unchanged.
\textit{MRT improves generalization by averaging within prompts, although prompt diversity must still come from more prompts. This unified picture explains why MRT is complementary to prompt collection rather than a substitute for it, and why its gains are largest in high response-diversity, low prompt-redundancy regimes.}

\begin{remark}[Bounded-loss interpretation]
\label{remark:bounded_loss}
The bounds are stated for bounded loss.
For cross-entropy, this can be interpreted via a clipped surrogate or per-token truncation.
Empirically, the results 
capture the concentration behavior of the loss, with the approximation gap governed by the upper tail.
\end{remark}

\section{Which Responses to Keep? Target-Aware Selection}
\label{sec:which_keep}

Sections~\ref{sec:when_mrt}--\ref{sec:how_much} show that increasing the retained
multiplicity \(K\) reduces response-level estimation error. We now separate this from a
second design choice: \emph{which} \(K\) responses to keep from an overgenerated pool
\(\mathcal Y'(\mathbf x)=\{\mathbf y_1,\ldots,\mathbf y_n\}\). This choice is not merely
an implementation detail. The selector determines the training target: it may preserve the
candidate distribution, tilt it toward high reward, or bias training toward broader coverage
of response modes.

\noindent \textbf{Distributional target: Random-\(K\)-of-\(N\).}
For distributional fine-tuning, the selector should preserve the conditional response
distribution. Random-\(K\)-of-\(N\) (RKoN) does exactly this.

\begin{proposition}[RKoN preserves the candidate distribution]
\label{prop:rkon}
Let
\(\mathbf y_1,\ldots,\mathbf y_n
\overset{\mathrm{i.i.d.}}{\sim}
q(\cdot\mid\mathbf x)\),
and let RKoN choose \(K\) candidates uniformly at random, independent of their values.
For \(g(\mathbf y)=\nabla_\theta\ell(\theta;\mathbf x,\mathbf y)\),
\(
    \mathbb E[\hat g_{\mathrm{RKoN}}]
    =
    \mathbb E_{\mathbf y\sim q(\cdot\mid\mathbf x)}[g(\mathbf y)],
\) and
\(
    \operatorname{Cov}(\hat g_{\mathrm{RKoN}})
    =
    \frac{1}{K}
    \operatorname{Cov}_{\mathbf y\sim q(\cdot\mid\mathbf x)}(g).
\)
In particular, if \(q=P\), RKoN is unbiased for the distributional target.
\end{proposition}

\noindent
Thus RKoN is the assumption-light default for studying MRT itself, because increasing \(K\) reduces gradient-estimator variance at the predicted \(1/K\) rate without changing the target
distribution.

\noindent \textbf{Reward target: Best-\(K\)-of-\(N\).}
Best-\(K\)-of-\(N\) (BKoN) instead selects the highest-reward responses. This can be the
right selector when the goal is reward optimization, but it is generally biased for
distributional fine-tuning. The bias does not vanish by increasing the candidate pool; in
fact, a larger pool can sharpen the reward-tail tilt.

\begin{proposition}[Persistent upper-tail bias]
\label{prop:bon_bias_main}
Suppose responses decompose into modes with probabilities \(p_m\) and distinct gradient
means \(\{\mu_m\}\). If one mode \(m^\star\) dominates the reward upper tail, then as
\(n\to\infty\),
\[
    \mathbb E[\hat g_{\mathrm{BKoN}}]
    \to
    \mu_{m^\star},
    \qquad
    \mu_{m^\star}
    \neq
    \sum_m p_m\mu_m
    =
    \bar g
\]
whenever the reward-dominant mode is not distributionally representative.
\end{proposition}
The limit is stated under bounded gradients and a tail-stable gradient mean for the dominant mode (see Appendix~\ref{app:bkon_bias} for proof and the formal sufficient conditions).

\noindent
This formalizes a target mismatch rather than a failure of optimization: BKoN trains on a
reward-tilted conditional distribution. Empirically, this appears as higher reward but lower
diversity and worse reference likelihood.

\noindent \textbf{Coverage-aware selection.}
RKoN is unbiased, but for small \(K\) it may spend multiple samples on the same response
mode. The ideal variance-reducing alternative is stratified sampling across modes, which
removes between-mode variance while preserving the distributional target. Since true modes
are unknown, we use diversity-aware selectors as practical approximations. In particular,
Diverse-Best-\(K\)-of-\(N\) (DBKoN) balances quality and coverage, while DKoN denotes the
pure-diversity special case. Our implementation, GRADES, greedily optimizes a monotone
facility-location surrogate over nonnegative quality scores and response similarities,
giving the standard \((1-1/e)\) approximation guarantee for that surrogate. The guarantee is
not a direct population-risk guarantee; it is useful only insofar as the similarity kernel
captures task-relevant response modes. Formal definitions, the stratification result, and
the GRADES proof are given in Appendix~\ref{app:selection_details}.

\noindent \textbf{Key Takeaway.}
The multiplicity \(K\) controls how much conditional information is retained, while the
selector controls \emph{which distribution} that information represents. RKoN tests the
variance--budget mechanism for distributional MRT; BKoN exposes reward-induced target
shift; and DBKoN/GRADES probes whether quality--coverage selection can recover some of the
variance reduction of stratification without collapsing valid modes.

\section{Experiments}
\label{sec:experiments}

Our theory makes three concrete predictions.
First, increasing response multiplicity \(K\) reduces only the within-prompt component of uncertainty, yielding diminishing returns once the prompt-level floor dominates.
Second, the benefit of MRT should be largest on datasets with high response-level variability relative to prompt-level variability.
Third, selectors change the \textit{training target}.

We test these predictions through four progressively less controlled validations.
A controlled synthetic setting first verifies the variance law, the cost-aware optimum, and selector bias under exact ground truth.
We then study real LLM fine-tuning across four multi-response benchmarks spanning instruction following, dialogue, preference, and code generation.
Next, we test whether dataset-level variance ratios predict where MRT helps most.
Finally, we examine selector behavior and implicit MRT through prompt redundancy.

\subsection{Controlled Validation}
\label{sec:exp_simulation}

Real LLM fine-tuning cannot directly observe \(V_{\mathbf x}\), \(V_{\mathbf y\mid\mathbf x}\), or selector bias.
We therefore first validate the theory in a setting where all quantities are known exactly by construction.

\noindent \textbf{Setup.}
Prompts \(\mathbf x\in\mathbb R^5\) are drawn from a standard Gaussian, and responses follow a three-mode conditional distribution
\[
    y
    =
    \sin(2x_1)+0.5x_2+\mathrm{offset}_m+\varepsilon,
    \qquad
    \mathrm{offset}_m\in\{-1.5,0,+1.5\},
    \qquad
    \varepsilon\sim\mathcal N(0,0.01),
\]
with modes sampled uniformly.
A reward model favors the central mode without making the others invalid, so reward optimization genuinely competes with distributional coverage.
We train mixture-density networks by NLL and evaluate population loss on fresh draws from the true conditional distribution \(P(y\mid\mathbf x)\).

\begin{figure}[t]
  \centering
  \begin{minipage}[t]{0.3\linewidth}
    \centering
    \includegraphics[width=\linewidth]{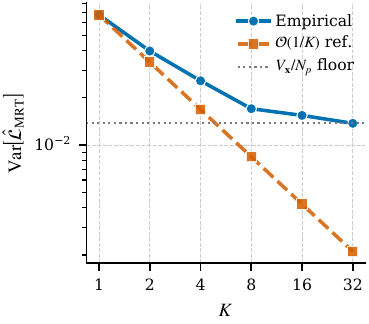}\\
    {\footnotesize (a) variance decomposition}
  \end{minipage}\hfill
  \begin{minipage}[t]{0.3\linewidth}
    \centering
    \includegraphics[width=\linewidth]{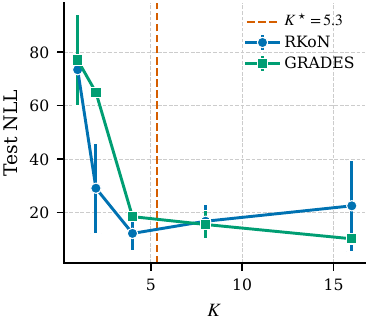}\\
    {\footnotesize (b) cost-aware budget \(K^\star\)}
  \end{minipage}\hfill
  \begin{minipage}[t]{0.3\linewidth}
    \centering
    \includegraphics[width=\linewidth]{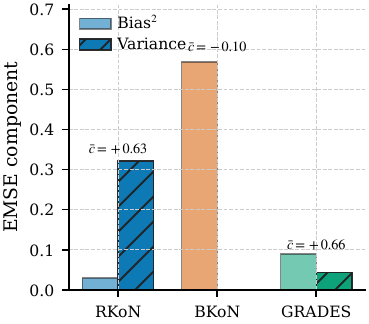}\\
    {\footnotesize (c) selector EMSE}
  \end{minipage}
  \caption{\textbf{Controlled validation under exact ground truth.}
  (a) Empirical variance follows the predicted
  \(V_{\mathbf x}/N_p+V_{\mathbf y\mid\mathbf x}/(N_pK)\) law and approaches the prompt floor.
  (b) Under a fixed prompt-plus-response budget, test NLL is minimized near the predicted
  \(K^\star=\sqrt{(V_{\mathbf y\mid\mathbf x}/V_{\mathbf x})(C_p/C_r)}\).
  (c) Selector EMSE decomposes into bias and variance: RKoN is unbiased but pays mode-missing variance, DBKoN reduces variance through coverage, and BKoN exhibits persistent reward-mode bias and negative gradient alignment.}
  \label{fig:sim_validation}
  \vspace{-10pt}
\end{figure}

\noindent \textbf{Results.}
Figure~\ref{fig:sim_validation}(a) validates Lemma~\ref{lemma:variance_decomposition}: increasing \(K\) reduces only the within-prompt component while leaving the prompt floor unchanged.
Figure~\ref{fig:sim_validation}(b) validates Prop.~\ref{prop:budget_allocation}: under a fixed budget \(\mathcal B=N_p(C_p+KC_r)\), too small \(K\) leaves response-level uncertainty unchecked, while too large \(K\) reduces the number of prompts enough to inflate the prompt floor, yielding the predicted U-shaped curve.
Finally, Figure~\ref{fig:sim_validation}(c) validates the target-aware view of \S\ref{sec:which_keep}: RKoN preserves the distributional target, BKoN shifts toward reward-dominant modes, and DBKoN partially recovers the variance reduction of stratification through coverage-aware selection.

Together, these experiments isolate the three core mechanisms of the theory:
the variance decomposition, the cost-aware optimum, and selector-induced target shift.
Additional controlled variants, including a discrete-token transformer and a BKoN candidate-pool sweep, are deferred to App.~\ref{app:simulation}.

\subsection{Large Language Model Evaluation}
\label{sec:exp_llm}

We now test whether the same mechanisms appear in realistic fine-tuning.
Unless stated otherwise, we fine-tune \texttt{Llama-3.1-8B-Instruct} with LoRA while varying only the response multiplicity \(K\) and the selector.
Training details, hyperparameters, and robustness checks are deferred to App.~\ref{app:setup}.

\noindent \textbf{Datasets and evaluation.}
We evaluate on four multi-response benchmarks:
Gold~\citep{coste2023reward} (instruction following),
Nectar~\citep{starling2023} (dialogue),
UltraFeedback~\citep{cui2023ultrafeedback} (preference-aligned instruction following),
and CodeContests~\citep{li2022competition} (competitive programming).
Gold and CodeContests provide the richest response multiplicity, while Nectar and UltraFeedback are comparatively prompt-floor dominated.

For each dataset we train on multiple responses per prompt and evaluate held-out \emph{reference loss}, i.e., the mean token-level cross-entropy on unseen reference responses from the dataset.
Lower reference loss indicates better fit to the underlying conditional response distribution.
We additionally report:
(i) \emph{reference coverage}, defined as the mean maximum cosine similarity between held-out references and the model's \(16\) sampled generations using \texttt{all-MiniLM-L6-v2};
(ii) \emph{semantic diversity}, measured as average pairwise embedding distance among sampled generations;
and (iii) reward and Pass@\(\!K\) only as target-shift diagnostics.
Complete dataset statistics and evaluation details are provided in App.~\ref{app:setup}.

\noindent \textbf{Does increasing \(K\) help in real LLM fine-tuning?}
The cleanest test of MRT is unbiased Random-\(K\)-of-\(N\) (RKoN), since it changes response multiplicity without changing the target distribution.
Figure~\ref{fig:cross_family} sweeps \(K\in\{1,2,4,8,16\}\) under RKoN on Gold across seven instruction-tuned backbones spanning the Llama-3, Qwen-3.5, and Gemma-4 families.

\begin{figure}[t]
  \centering
  \includegraphics[width=0.7\linewidth]{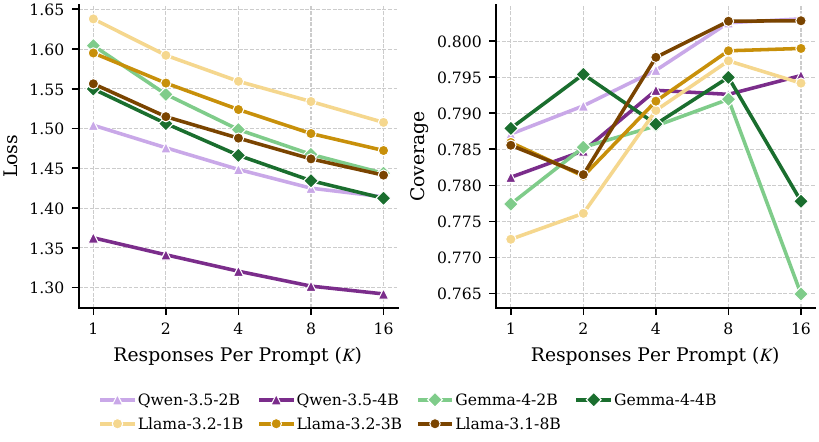}
  \caption{\textbf{Increasing response multiplicity improves distributional fit across model families.}
  Under unbiased RKoN on Gold, reference loss decreases monotonically with \(K\) for all seven backbones.
  Coverage generally improves and then saturates, matching the diminishing-returns regime predicted once the prompt floor dominates.}
  \label{fig:cross_family}
  \vspace{-10pt}
\end{figure}

Reference loss decreases monotonically with \(K\) for every model, with best-\(K\) improvements of \(0.071\)--\(0.161\) NLL over standard single-response training.
Coverage likewise improves and then saturates around \(K=8\)--\(16\), exactly the regime predicted once the response-level component has been averaged down.
The same qualitative pattern appears across all three model families.

\noindent \textbf{Does the variance ratio predict where MRT helps?}
The variance--budget law makes a stronger prediction than simply ``larger \(K\) helps'':
datasets with larger response-level variability relative to prompt-level variability should benefit more from MRT.

Table~\ref{tab:variance_predicts_gain} tests this directly.
We estimate \(V_{\mathbf x}\) and \(V_{\mathbf y\mid\mathbf x}\) by decomposing per-response NLL under the fixed base \texttt{Llama-3.1-8B-Instruct} checkpoint, compute the implied
\(K^\star_c=\sqrt{(V_{\mathbf y\mid\mathbf x}/V_{\mathbf x})c}\),
and compare these quantities against the observed improvement of R2oN over R1oN.

\begin{table}[t]
  \centering
  \small
  \caption{\textbf{Variance ratios predict where MRT helps most.}
  \(N\) denotes the number of candidate responses per prompt.
  Datasets with larger response-level variability relative to prompt-level variability exhibit substantially larger MRT gains under the same R1oN \(\to\) R2oN intervention.}
  \setlength{\tabcolsep}{5.5pt}
  \renewcommand{\arraystretch}{1.08}
  \begin{tabular}{lcccccc}
    \toprule
    Dataset & Domain & \(N\) &
    \(V_{\mathbf y\mid\mathbf x}/V_{\mathbf x}\) &
    \(K^\star_{10}\) &
    R1oN \(\to\) R2oN &
    MRT gain \\
    \midrule
    CodeContests
      & Code
      & 64
      & 1.102
      & 3.32
      & \(0.792\to0.712\)
      & \(0.080\) \;(\(10.1\%\)) \\
    Gold
      & Instruction
      & 64
      & 0.751
      & 2.74
      & \(1.556\to1.515\)
      & \(0.041\) \;(\(2.6\%\)) \\
    UltraFeedback
      & Preference
      & 4
      & 0.136
      & 1.17
      & \(1.052\to1.038\)
      & \(0.014\) \;(\(1.3\%\)) \\
    Nectar
      & Dialogue
      & 7
      & 0.122
      & 1.10
      & \(0.909\to0.893\)
      & \(0.016\) \;(\(1.8\%\)) \\
    \bottomrule
  \end{tabular}
  \label{tab:variance_predicts_gain}
  \vspace{-6pt}
\end{table}

The trend is clear.
CodeContests exhibits the largest response-level share and correspondingly the largest MRT gain.
Gold follows the same pattern at smaller scale.
Nectar and UltraFeedback are strongly prompt-floor dominated and show much smaller improvements.
Thus the variance ratio correctly predicts the operating regime in which MRT is most valuable.
The same matched intervention also improves over a single best-response baseline: R2oN achieves lower reference loss than both R1oN and B1oN on all four datasets (Table~\ref{tab:fixed_rows}).
On CodeContests, R2oN additionally gives the best Pass@16, indicating that the distributional gain can transfer to an execution-based metric.

\noindent \textbf{Does selection change the target?}
We next fix the dataset and vary only the selector.
Figure~\ref{fig:gold_grouped_bars} reports the full Gold grid across RKoN, BKoN, DKoN, and DBKoN.

\begin{figure}[t]
\centering
\includegraphics[width=\textwidth]{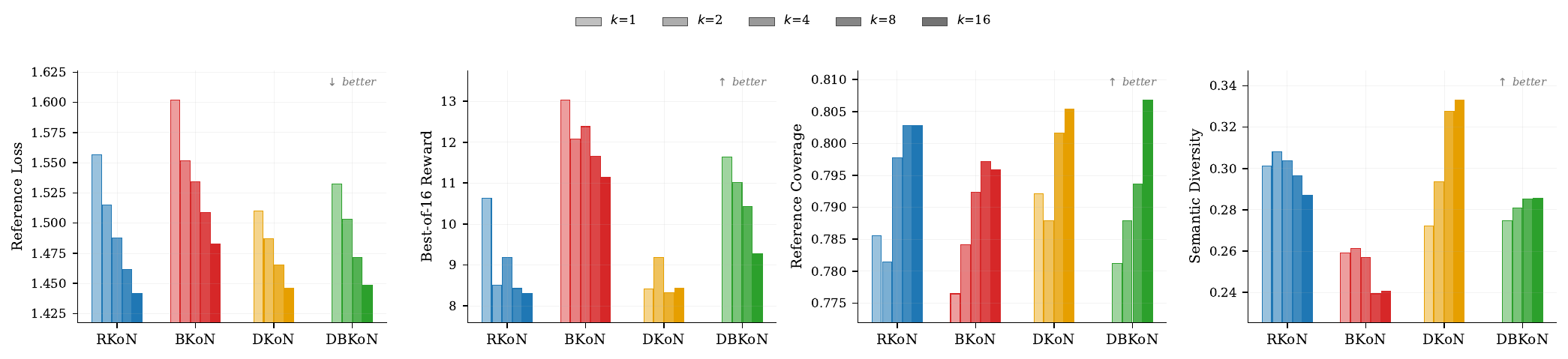}
\caption{\textbf{Selection changes the target.}
Within a selector, increasing \(K\) generally improves distributional metrics.
Across selectors, the target-aware decomposition emerges clearly:
RKoN best fits references,
BKoN maximizes reward but collapses diversity,
DKoN maximizes diversity,
and DBKoN provides the strongest reference coverage through coverage-aware selection.}
\label{fig:gold_grouped_bars}
\vspace{-10pt}
\end{figure}

The within-selector trend validates the role of \(K\): increasing response multiplicity generally lowers reference loss and improves coverage.
The across-selector trend validates the target-aware analysis of \S\ref{sec:which_keep}:
RKoN gives the best reference likelihood,
BKoN achieves the highest reward but the worst diversity and likelihood,
DKoN maximizes diversity,
and DBKoN gives the strongest reference coverage by balancing quality and coverage.
Thus selectors do not merely affect optimization efficiency; they alter which distribution the model is trained toward.

\noindent \textbf{Implicit MRT through prompt redundancy.}
Finally, we test the implicit-MRT prediction from \S\ref{sec:when_mrt}.
For this experiment we use \textsc{Mosaic-1k}, a controlled benchmark containing four paraphrases and four valid responses per base item, allowing prompt multiplicity and response multiplicity to be varied independently under fixed supervision budget.

\begin{table}[t]
  \centering
  \setlength{\tabcolsep}{8pt}
  \caption{\textbf{Prompt redundancy acts like MRT only when it exposes distinct responses.}
  At equal supervised-pair budget, paraphrase redundancy paired with distinct responses matches explicit MRT, while paraphrases paired with near-duplicate responses are substantially weaker.}
  \renewcommand{\arraystretch}{1.08}
  \begin{tabular}{lcc}
    \toprule
    Method & Loss \(\downarrow\) & Coverage \(\uparrow\) \\
    \midrule
    SRT \,(\(1{\times}1\))                       & 1.691 & 0.707 \\
    MRT \,(\(1{\times}4\))                       & 1.622 & 0.713 \\
    Implicit MRT, Diverse\,(\(4{\times}1\))      & 1.616 & 0.712 \\
    Implicit MRT, Similar\,(\(4{\times}1\))      & 1.630 & 0.714 \\
    Upper Bound \,(\(4{\times}4\))               & 1.566 & 0.715 \\
    \bottomrule
  \end{tabular}
  \label{tab:mosaic_implicit}
  \vspace{-10pt}
\end{table}

The key comparison is between the two \(4{\times}1\) conditions.
Both expose the model to four paraphrased prompts, but only the Diverse condition pairs them with distinct valid responses.
Diverse matches explicit MRT almost exactly, whereas Similar lags behind.
This isolates the mechanism behind implicit MRT:
prompt redundancy substitutes for explicit response multiplicity only when paraphrases expose additional conditional modes.

The experiments consistently support the same picture across controlled simulation and real fine-tuning.
Increasing \(K\) reduces response-level uncertainty but exhibits diminishing returns once the prompt floor dominates.
The variance ratio predicts where MRT is most valuable.
Selectors determine the training target itself:
unbiased selection improves distributional fit, reward-only selection collapses toward upper-tail modes, and coverage-aware selection interpolates between the two.

\section{Conclusion}
\label{sec:conclusion}

This paper argues that the value of post-training supervision depends not only on scale, but on how supervision is allocated across sources of uncertainty. Prompts and responses play different statistical roles: new prompts reduce uncertainty about the input distribution, while additional responses reduce uncertainty about the conditional output distribution. This distinction yields a variance--budget law for response multiplicity, and offers guidance on how MRT is most useful when within-prompt response variation is large, prompts are costly or non-redundant, and additional responses are cheap.

The resulting prescription is simple. Estimate the variance ratio \(V_{\mathbf y\mid\mathbf x}/V_{\mathbf x}\) on a proxy model, choose the response multiplicity according to the resulting operating regime, and use a selector whose target matches the downstream objective. RKoN preserves the candidate distribution and isolates the distributional benefit of MRT; BKoN optimizes a reward-tilted target and can collapse toward upper-tail modes; coverage-aware selectors trade reward concentration against mode coverage. Thus the number of retained responses controls how much conditional information is used, while the selector controls which distribution that information represents.

Across controlled simulations and LLM fine-tuning experiments, we find the same pattern: MRT gives the largest gains in high response-diversity, low prompt-redundancy regimes, with diminishing returns once the prompt-level floor dominates. The same framework also explains implicit MRT, where redundant prompts can substitute for explicit response multiplicity when they expose distinct conditional modes. Overall, MRT is valuable not because it adds more tokens, but because it exposes more of the conditional structure the model is trying to learn.

\bibliographystyle{plainnat}
\bibliography{refs}

\newpage
\appendix

\section{Extended Related Work}
\label{sec:app_related}

\paragraph{Response multiplicity in post-training and inference.}
Modern LLM pipelines already exploit multiple responses per prompt.
Instruction tuning, rejection-sampled fine-tuning, reward-model filtering, RLHF-style pipelines, and iterative self-training often generate, rank, or reuse multiple candidate completions during training~\citep{bai2022constitutional,touvron2023llama,dubey2024llama,gulcehre2023reinforced,dong2023raft,yuan2023rrhf,zelikman2022star}.
At inference time, pass@$K$, self-consistency, process verifiers, and test-time scaling similarly improve performance by sampling multiple candidate solutions before selecting or aggregating them~\citep{chen2021evaluating,wang2023selfconsistency,lightman2023lets,snell2024scaling,brown2024monkeys,wu2024inference}.
MRT studies the corresponding training-time data-allocation question: when a prompt admits multiple plausible responses, how many should be retained for fine-tuning, and by what selection rule?

\paragraph{Best-of-\(N\), reward overoptimization, and distributional coverage.}
Best-of-\(N\) selection is effective when the target is reward maximization, but it changes the training distribution.
Prior work studies best-of-\(N\) alignment policies, reward overoptimization, reward hacking, and broader limitations of RLHF-style objectives~\citep{beirami2024theoretical,gao2023scaling,coste2023reward,eisenstein2024helping,casper2023open}.
Our analysis gives a training-time counterpart: under a reward-tail dominance condition, reward-only selection can concentrate on reward-dominant modes and leave other valid modes uncovered.
This matters for supervised fine-tuning under cross-entropy, a strictly proper scoring rule~\citep{gneiting2007strictly}, where the distributional target is not simply the highest-reward completion.

\paragraph{Instruction data curation and response allocation.}
Instruction-tuning research has largely focused on which prompts to collect, synthesize, filter, or reweight~\citep{wei2022finetuned,wang2023selfinstruct,zhou2024lima,xu2024magpie,lu2024instag,liu2024deita,song2024scaling,sun2024principle}.
MRT asks an orthogonal question: given a prompt, how many responses should be kept?
Our variance-budget law makes the interaction explicit by separating prompt-level and response-level uncertainty.
It also explains why massive or highly redundant corpora may partially mimic MRT through near-equivalent prompts, while curated non-redundant datasets can benefit more from explicit response multiplicity.

\paragraph{Hierarchical generalization, variance reduction, and diversity selection.}
Classical generalization analyses usually assume i.i.d. examples~\citep{shalevshwartz2014understanding,mohri2018foundations}, while multi-task and hierarchical Rademacher analyses study shared structure across related tasks~\citep{baxter2000model,maurer2016benefit,tripuraneni2021provable,du2021fewshot}.
MRT induces a different hierarchy: prompts are sampled across tasks, and responses are conditionally sampled within prompts.
Our analysis adapts variance-sensitive and conditional Rademacher tools to this response-stratified setting.
For response selection, GRADES builds on the long line of submodular facility-location and diversity-aware subset selection methods~\citep{nemhauser1978analysis,minoux1978accelerated,lin2011class,kirchhoff2014submodularity,wei2015submodularity,sener2018coreset,xia2024less,agarwal2024delift}, but applies them within each prompt's candidate pool to balance quality with coverage over valid response modes.
\section{Theoretical Appendix}
\label{app:theory}

\subsection{Preliminaries and Concentration Tools}
\label{app:prelim}

\paragraph{Notation recap.}
Throughout, prompts $\mathbf x_1,\dots,\mathbf x_{N_p}$ are drawn i.i.d.\ from a marginal
$P(\mathbf x)$, and for each prompt $\mathbf x_i$ we draw $K$ responses
$\mathbf y_{i,1},\dots,\mathbf y_{i,K}$ i.i.d.\ from the conditional
$P(\mathbf y\mid\mathbf x_i)$.
The per-example loss $\ell(\theta;\mathbf x,\mathbf y)$ is assumed bounded,
$\ell\in[0,B]$ (Remark~\ref{remark:bounded_loss}; see also
Remark~\ref{remark:clipped_loss_app}).
The MRT empirical risk is
\[
    \widehat{\mathcal L}_{\mathrm{MRT}}(\theta)
    =
    \frac{1}{N_p}\sum_{i=1}^{N_p}
    \frac{1}{K}\sum_{j=1}^{K}
    \ell(\theta;\mathbf x_i,\mathbf y_{ij}),
\]
and we write $g_\theta(\mathbf x)\coloneqq\mathbb E_{\mathbf y\mid\mathbf x}\ell(\theta;\mathbf x,\mathbf y)$
for the prompt-level conditional risk, so that
$\mathcal L(\theta)=\mathbb E_{\mathbf x}g_\theta(\mathbf x)$.
Recall the variance components
$V_{\mathbf x}(\theta)=\operatorname{Var}_{\mathbf x}[g_\theta(\mathbf x)]$ and
$V_{\mathbf y\mid\mathbf x}(\theta)=\mathbb E_{\mathbf x}\!\bigl[\operatorname{Var}_{\mathbf y\mid\mathbf x}\ell\bigr]$.
Note that $g_\theta(\mathbf x)\in[0,B]$ as a conditional expectation of $\ell\in[0,B]$.
We write $\mathcal G=\{\mathbf x\mapsto g_\theta(\mathbf x):\theta\in\Theta\}$ for the
prompt-level loss class and define the prompt-level Rademacher complexity
\[
    \mathfrak R_{N_p}(\mathcal G)
    =\mathbb E_{X,\sigma}\!\left[\sup_{\theta\in\Theta}\frac{1}{N_p}\sum_{i=1}^{N_p}\sigma_i\,g_\theta(\mathbf x_i)\right],
\]
where $\sigma_i\in\{\pm 1\}$ are i.i.d.\ Rademacher signs independent of $X$. The
conditional response-level complexity $\mathfrak R_{N_pK}(\mathcal F\mid X)$ is as
defined in the main text (Section~\ref{sec:how_much}).

\begin{lemma}[McDiarmid's inequality \citep{mcdiarmid1989method}]
\label{lemma:mcdiarmid}
Let $X_1,\dots,X_m$ be independent random variables, and let
$f:\mathcal X^m\to\mathbb R$ satisfy the bounded-differences property
\[
    \sup_{x_1,\dots,x_m,x_i'}
    \bigl|f(x_1,\dots,x_i,\dots,x_m)-f(x_1,\dots,x_i',\dots,x_m)\bigr|\le c_i,
    \qquad i=1,\dots,m.
\]
Then for every $\epsilon>0$,
\[
    \mathbb P\!\bigl(f(X_1,\dots,X_m)-\mathbb E[f]\ge\epsilon\bigr)
    \le
    \exp\!\left(-\frac{2\epsilon^2}{\sum_{i=1}^m c_i^2}\right).
\]
\end{lemma}

\begin{lemma}[Lifting conditional tail bounds]
\label{lemma:lift}
Let $E$ be an event and $X$ a random variable. If
$\mathbb P(E\mid X)\ge 1-\delta$ almost surely, then $\mathbb P(E)\ge 1-\delta$.
\end{lemma}
\begin{proof}
$\mathbb P(E^c)=\mathbb E[\mathbb P(E^c\mid X)]\le\delta$.
\end{proof}

\begin{remark}[Bounded-loss surrogate, formalizing Remark~\ref{remark:bounded_loss}]
\label{remark:clipped_loss_app}
Cross-entropy is unbounded, so the bounded-loss assumption is enforced through a
truncation. Fix a per-token clip level $B_{\mathrm{tok}}>0$ and define
\[
    \tilde\ell_t(\theta;\mathbf x,y_{<t})
    \coloneqq
    \min\!\bigl(-\log\pi_\theta(y_t\mid\mathbf x,y_{<t}),\;B_{\mathrm{tok}}\bigr),
    \qquad
    \tilde\ell(\theta;\mathbf x,\mathbf y)=\sum_{t=1}^T\tilde\ell_t,
\]
so that $\tilde\ell\in[0,T B_{\mathrm{tok}}]$ with $B\coloneqq T B_{\mathrm{tok}}$.
The associated approximation gap is bounded by
$\mathbb E\!\bigl[\,T\,(\ell_t-B_{\mathrm{tok}})^+\bigr]$, which decays whenever the
token negative log-likelihood has a sufficiently light upper tail. For typical
fine-tuning checkpoints this gap is empirically negligible at $B_{\mathrm{tok}}$ in the
high single digits, and our proofs below are stated for $\tilde\ell$ but read
identically for $\ell$ once truncation is applied.
\end{remark}

\begin{assumption}[Loss-class capacity]
\label{assum:capacity}
The loss class $\mathcal F=\{(\mathbf x,\mathbf y)\mapsto\ell(\theta;\mathbf x,\mathbf y):\theta\in\Theta\}$
satisfies one of the following standard capacity-control conditions:
\begin{enumerate}[label=(\roman*),leftmargin=1.6em]
\item \emph{Finite VC dimension.} $\mathcal F$ has VC subgraph dimension
$d_{\mathrm{VC}}<\infty$.
\item \emph{Lipschitz parametrization.} $\theta\mapsto\ell(\theta;\cdot)$ is
$L$-Lipschitz on a parameter set of effective dimension $d_{\mathrm{eff}}$ with bounded
diameter (e.g., LoRA adapters of rank $r$ with bounded Frobenius norm).
\item \emph{Dudley integrability.} The covering numbers
$\mathcal N(\varepsilon,\mathcal F,L_2(\hat P))$ satisfy
$\int_0^B\sqrt{\log\mathcal N(\varepsilon,\mathcal F,L_2(\hat P))}\,d\varepsilon<\infty$
for any empirical measure $\hat P$.
\end{enumerate}
Any of (i)--(iii) is sufficient for the conditional Rademacher complexity to satisfy
$\mathfrak R_{N_pK}(\mathcal F\mid X)=O\!\bigl(\sqrt{d_{\mathrm{eff}}/(N_pK)}\bigr)$ via
Dudley's entropy integral.
\end{assumption}

\subsection{Proof of Lemma~\ref{lemma:variance_decomposition} (Variance Decomposition)}
\label{app:variance_proof}

\begin{proof}
By the law of total variance applied to $\widehat{\mathcal L}_{\mathrm{MRT}}(\theta)$ with
respect to the prompt randomness $X=(\mathbf x_1,\dots,\mathbf x_{N_p})$,
\[
\operatorname{Var}\!\bigl[\widehat{\mathcal L}_{\mathrm{MRT}}(\theta)\bigr]
=\operatorname{Var}_X\!\bigl[\mathbb E_{Y\mid X}\widehat{\mathcal L}_{\mathrm{MRT}}\bigr]
+\mathbb E_X\!\bigl[\operatorname{Var}_{Y\mid X}\widehat{\mathcal L}_{\mathrm{MRT}}\bigr].
\]

\textit{First term.} Conditional on $X$, the inner average over the $K$ responses is
$g_\theta(\mathbf x_i)$ in expectation, so
$\mathbb E_{Y\mid X}\widehat{\mathcal L}_{\mathrm{MRT}}=\frac{1}{N_p}\sum_i g_\theta(\mathbf x_i)$.
Since the prompts are i.i.d.,
\[
    \operatorname{Var}_X\!\left[\frac{1}{N_p}\sum_{i=1}^{N_p}g_\theta(\mathbf x_i)\right]
    =\frac{1}{N_p}\operatorname{Var}_{\mathbf x}[g_\theta(\mathbf x)]
    =\frac{V_{\mathbf x}}{N_p}.
\]

\textit{Second term.} Conditional on $X$, the responses $\{\mathbf y_{ij}\}$ are independent
across $(i,j)$, so
\[
    \operatorname{Var}_{Y\mid X}\widehat{\mathcal L}_{\mathrm{MRT}}
    =\frac{1}{N_p^2}\sum_{i=1}^{N_p}\operatorname{Var}_{Y_i\mid\mathbf x_i}\!\left[\frac{1}{K}\sum_{j=1}^K\ell(\theta;\mathbf x_i,\mathbf y_{ij})\right]
    =\frac{1}{N_p^2 K}\sum_{i=1}^{N_p}\operatorname{Var}_{\mathbf y\mid\mathbf x_i}\ell.
\]
Taking the expectation over $X$,
\[
    \mathbb E_X\!\bigl[\operatorname{Var}_{Y\mid X}\widehat{\mathcal L}_{\mathrm{MRT}}\bigr]
    =\frac{1}{N_p K}\,\mathbb E_{\mathbf x}\!\bigl[\operatorname{Var}_{\mathbf y\mid\mathbf x}\ell\bigr]
    =\frac{V_{\mathbf y\mid\mathbf x}}{N_p K}.
\]
Summing the two contributions yields the claim.
\end{proof}

\noindent\emph{Why a $V_{\mathbf x}/N_p$ floor is irreducible.}
Prompts are sampled marginally i.i.d.\ (giving the $1/N_p$ factor on $V_{\mathbf x}$),
while responses are only \emph{conditionally} i.i.d.\ given prompts (giving the additional
$1/K$ factor on $V_{\mathbf y\mid\mathbf x}$). Had the responses been marginally i.i.d.,
both terms would scale as $1/(N_pK)$ and the floor would disappear; it is precisely the
hierarchical structure of MRT data that creates the irreducible prompt-level term.

\subsection{Proof of Proposition~\ref{prop:budget_allocation} (Cost-Aware Optimal Multiplicity)}
\label{app:budget_proof}

\begin{proof}
The budget constraint $\mathcal B=N_p(C_p+K C_r)$ gives $N_p=\mathcal B/(C_p+K C_r)$.
Substituting into the variance from Lemma~\ref{lemma:variance_decomposition},
\[
    V(K)
    \coloneqq
    \operatorname{Var}\!\bigl[\widehat{\mathcal L}_{\mathrm{MRT}}\bigr]
    =\frac{V_{\mathbf x}+V_{\mathbf y\mid\mathbf x}/K}{N_p}
    =\frac{(V_{\mathbf x}+V_{\mathbf y\mid\mathbf x}/K)(C_p+K C_r)}{\mathcal B}.
\]
Expanding,
\[
    V(K)=\frac{1}{\mathcal B}\!\left(V_{\mathbf x}C_p+V_{\mathbf y\mid\mathbf x}C_r
    +V_{\mathbf x}C_r\,K+\frac{V_{\mathbf y\mid\mathbf x}C_p}{K}\right).
\]
The first two terms are constant in $K$; the remaining sum
$V_{\mathbf x}C_r\,K+V_{\mathbf y\mid\mathbf x}C_p/K$ is strictly convex in $K>0$, since it
is a positive linear combination of $K$ and $1/K$.
Setting $dV/dK=0$,
\[
    V_{\mathbf x}C_r-\frac{V_{\mathbf y\mid\mathbf x}C_p}{K^2}=0
    \;\Longleftrightarrow\;
    K^\star=\sqrt{\frac{V_{\mathbf y\mid\mathbf x}}{V_{\mathbf x}}\cdot\frac{C_p}{C_r}}.
\]
By strict convexity, $K^\star$ is the unique global minimizer on $(0,\infty)$.
Comparing with $K=1$, MRT strictly improves over single-response training in the
continuous relaxation iff $K^\star>1$, equivalently
$V_{\mathbf y\mid\mathbf x}/V_{\mathbf x}>C_r/C_p$.
\end{proof}

\paragraph{Integer-$K$ remark.}
The continuous threshold above is necessary but not sufficient for the integer choice
$K{=}2$ to strictly improve over $K{=}1$. Only the $K$-dependent terms differ between the
two, so
\[
    V(2)-V(1)
    =\frac{1}{\mathcal B}\!\left(V_{\mathbf x}C_r-\frac{V_{\mathbf y\mid\mathbf x}C_p}{2}\right),
\]
which gives
\[
    V(2)<V(1)
    \;\Longleftrightarrow\;
    \frac{V_{\mathbf y\mid\mathbf x}}{V_{\mathbf x}}>\frac{2C_r}{C_p}.
\]
The factor $2$ is sharp: at the boundary, $V(2)=V(1)$.

\paragraph{Estimating $K^\star$ in practice.}
The two variance components can be estimated from a small validation pool. Take
$N_{\mathrm{val}}\in[50,100]$ prompts and generate $K_{\mathrm{val}}\in[8,16]$ responses
per prompt at a reference checkpoint. Let $\hat g_i=\frac{1}{K_{\mathrm{val}}}\sum_j\ell(\theta;\mathbf x_i,\mathbf y_{ij})$ and
$\hat s_i^2=\frac{1}{K_{\mathrm{val}}-1}\sum_j(\ell-\hat g_i)^2$. Then
\[
    \widehat V_{\mathbf y\mid\mathbf x}
    =\frac{1}{N_{\mathrm{val}}}\sum_i\hat s_i^2,
    \qquad
    \widehat V_{\mathbf x}
    =\max\!\left(0,\;
        \frac{1}{N_{\mathrm{val}}-1}\sum_i\bigl(\hat g_i-\bar{\hat g}\bigr)^2-\frac{\widehat V_{\mathbf y\mid\mathbf x}}{K_{\mathrm{val}}}
    \right),
\]
where the bias correction subtracts the within-prompt contribution that leaks into the
empirical between-prompt variance. Plugging into $K^\star$ gives a simple plug-in rule
that is robust enough to be used as a budget diagnostic.

\subsection{Proof of Theorem~\ref{theorem:bernstein_bound} (Bernstein Bound)}
\label{app:bernstein_proof}

\begin{proof}
Define per-prompt deviations
\[
    Z_i=\hat\ell(\theta;\mathbf x_i)-\mathcal L(\theta),
    \qquad
    \hat\ell(\theta;\mathbf x_i)=\frac{1}{K}\sum_{j=1}^K\ell(\theta;\mathbf x_i,\mathbf y_{ij}),
\]
so that $\widehat{\mathcal L}_{\mathrm{MRT}}-\mathcal L=\frac{1}{N_p}\sum_i Z_i$.
The $Z_i$ are independent: their randomness comes from the disjoint blocks
$(\mathbf x_i,\mathbf y_{i,1:K})$, which are i.i.d.\ across $i$. They are also bounded,
$|Z_i|\le B$ since $\ell\in[0,B]$, and have mean zero. By
Lemma~\ref{lemma:variance_decomposition} applied to a single block,
\[
    \sigma^2\coloneqq\operatorname{Var}(Z_i)=V_{\mathbf x}+\frac{V_{\mathbf y\mid\mathbf x}}{K}.
\]
Bernstein's inequality applied to $\bar Z=\frac{1}{N_p}\sum_i Z_i$ yields, for every $\epsilon>0$,
\[
    \mathbb P(|\bar Z|\ge\epsilon)
    \le 2\exp\!\left(-\frac{N_p\epsilon^2}{2\sigma^2+\tfrac{2}{3}B\epsilon}\right).
\]
Setting the right-hand side equal to $\delta$ gives the quadratic
$N_p\epsilon^2-\tfrac{2B\ln(2/\delta)}{3}\epsilon-2\sigma^2\ln(2/\delta)=0$, whose
positive root is
\[
    \epsilon=\frac{\tfrac{2B\ln(2/\delta)}{3}+\sqrt{\bigl(\tfrac{2B\ln(2/\delta)}{3}\bigr)^2+8N_p\sigma^2\ln(2/\delta)}}{2N_p}.
\]
Using $\sqrt{a^2+b}\le a+\sqrt b$ for $a,b\ge 0$,
\[
    \epsilon\le\sqrt{\frac{2\sigma^2\ln(2/\delta)}{N_p}}+\frac{2B\ln(2/\delta)}{3N_p}
    =\sqrt{\frac{2\bigl(V_{\mathbf x}+V_{\mathbf y\mid\mathbf x}/K\bigr)\ln(2/\delta)}{N_p}}+\frac{2B\ln(2/\delta)}{3N_p},
\]
proving the claim.
\end{proof}

\paragraph{Clipped-loss corollary.}
Applying the same proof to the clipped surrogate
$\tilde\ell\in[0,B]$ of Remark~\ref{remark:clipped_loss_app} gives the bound for
$\tilde{\mathcal L}$ and $\widehat{\tilde{\mathcal L}}_{\mathrm{MRT}}$. Translating back
to unclipped cross-entropy adds a $\theta$-dependent additive bias
$\Delta_{\mathrm{clip}}(\theta)\coloneqq|\mathcal L(\theta)-\tilde{\mathcal L}(\theta)|\le\mathbb E\!\bigl[\sum_{t=1}^T(\ell_t(\theta)-B_{\mathrm{tok}})^+\bigr]$,
which is small whenever the per-token negative-log-likelihood has a sufficiently light
upper tail at the operating checkpoint. The same comment applies verbatim to the
hierarchical Rademacher bound below, with $\sup_\theta\Delta_{\mathrm{clip}}(\theta)$
replacing the per-$\theta$ bias on the uniform side.

\paragraph{Tightening factor.}
Single-response training corresponds to $K{=}1$ in the bracketed standard deviation,
giving $V_{\mathbf x}+V_{\mathbf y\mid\mathbf x}$. MRT replaces this by
$V_{\mathbf x}+V_{\mathbf y\mid\mathbf x}/K$, so the leading term of the bound is tightened
by the factor
\[
    \sqrt{\frac{V_{\mathbf x}+V_{\mathbf y\mid\mathbf x}}{V_{\mathbf x}+V_{\mathbf y\mid\mathbf x}/K}},
\]
which can be substantial whenever $V_{\mathbf y\mid\mathbf x}\gg V_{\mathbf x}$---precisely
the high within-prompt-variability regime in which MRT is most valuable. As $K\to\infty$,
the factor saturates at $\sqrt{1+V_{\mathbf y\mid\mathbf x}/V_{\mathbf x}}$, recovering the
prompt-floor interpretation.

\subsubsection{Generator Mismatch}
\label{app:generator_mismatch}

In practice, responses are drawn from a generator $q(\mathbf y\mid\mathbf x)$ rather than
from the target conditional $P(\mathbf y\mid\mathbf x)$. We quantify the resulting bias.

\begin{proposition}[Generator mismatch]
\label{prop:generator_mismatch}
Suppose responses are drawn from $q$ with
$\mathrm{TV}\!\bigl(q(\cdot\mid\mathbf x),P(\cdot\mid\mathbf x)\bigr)\le\delta_{\mathrm{TV}}$
uniformly in $\mathbf x$, and let $\mathcal L_P$ denote the population risk under the
target $P$. Then
\[
    \bigl|\mathbb E_q[\widehat{\mathcal L}_{\mathrm{MRT}}]-\mathcal L_P\bigr|
    \le B\,\delta_{\mathrm{TV}}.
\]
\end{proposition}
\begin{proof}
For any fixed $\theta$ and prompt $\mathbf x$, the bound $\ell\in[0,B]$ and the
variational definition of total variation give
\[
    \bigl|\mathbb E_q\ell-\mathbb E_P\ell\bigr|
    \le B\cdot\mathrm{TV}\!\bigl(q(\cdot\mid\mathbf x),P(\cdot\mid\mathbf x)\bigr)
    \le B\,\delta_{\mathrm{TV}}.
\]
Marginalizing over $\mathbf x$ and observing that
$\mathbb E_q[\widehat{\mathcal L}_{\mathrm{MRT}}]=\mathbb E_{\mathbf x}\mathbb E_q\ell$
yields the claim.
\end{proof}

The same Bernstein concentration around $\mathbb E_q[\widehat{\mathcal L}_{\mathrm{MRT}}]$
goes through unchanged with $q$-distributed responses; converting to the target $P$ only
adds the additive bias above.

\subsection{Proof of Theorem~\ref{theorem:rademacher_bound} (Hierarchical Uniform Bound)}
\label{app:rademacher_proof}

\noindent\emph{Proof strategy.}
The challenge is to handle two layers of randomness---$N_p$ i.i.d.\ prompts and $K$
conditionally i.i.d.\ responses per prompt---while keeping $\sup_\theta$ outside both
sums so the response-level Rademacher term benefits from all $N_pK$ samples jointly. The
proof decomposes the gap into a prompt-level term (controlled by McDiarmid over prompts)
and a response-level term (conditional symmetrization, then McDiarmid unconditionally),
with a third McDiarmid step to control fluctuations of the conditional Rademacher
complexity itself.

\begin{proof}
Define the prompt-level empirical risk
$\bar{\mathcal L}(\theta)=\frac{1}{N_p}\sum_{i=1}^{N_p}g_\theta(\mathbf x_i)$.
Decompose
\[
    \sup_\theta\!\bigl(\mathcal L(\theta)-\widehat{\mathcal L}_{\mathrm{MRT}}(\theta)\bigr)
    \le
    \underbrace{\sup_\theta\!\bigl(\mathcal L(\theta)-\bar{\mathcal L}(\theta)\bigr)}_{A(X)}
    +
    \underbrace{\sup_\theta\!\bigl(\bar{\mathcal L}(\theta)-\widehat{\mathcal L}_{\mathrm{MRT}}(\theta)\bigr)}_{C(X,Y)}.
\]

\textit{Step 1: Prompt-level term $A(X)$.}
$A(X)$ is a standard empirical process over $N_p$ i.i.d.\ prompts with
$g_\theta\in[0,B]$. Changing one prompt changes $A(X)$ by at most $B/N_p$, so by
McDiarmid (Lemma~\ref{lemma:mcdiarmid}) with probability $\ge 1-\delta/4$,
\[
    A(X)\le\mathbb E[A(X)]+B\sqrt{\frac{\ln(4/\delta)}{2N_p}}.
\]
Standard symmetrization over $\mathcal G$ \citep{bartlett2002rademacher} gives
$\mathbb E[A(X)]\le 2\,\mathfrak R_{N_p}(\mathcal G)$.

\textit{Step 2: Response-level term $C(X,Y)$.}
Conditioning on $X$, $C(X,Y)$ depends on the $N_pK$ independent responses
$\mathbf y_{ij}$. Changing a single $\mathbf y_{ij}$ alters $C$ by at most $B/(N_pK)$.
McDiarmid conditionally on $X$ gives, with probability $\ge 1-\delta/4$ (lifted to
unconditional probability via Lemma~\ref{lemma:lift}),
\[
    C(X,Y)\le\mathbb E[C\mid X]+B\sqrt{\frac{\ln(4/\delta)}{2N_pK}}.
\]
The conditional expectation is bounded by symmetrization with ghost responses
$\mathbf y_{ij}'\sim P(\cdot\mid\mathbf x_i)$ drawn independently of $\mathbf y_{ij}$.
Writing $\Phi_\theta(Y)=\frac{1}{N_pK}\sum_{i,j}\ell(\theta;\mathbf x_i,\mathbf y_{ij})$
and using $\bar{\mathcal L}(\theta)=\mathbb E_{Y'\mid X}\Phi_\theta(Y')$,
Jensen and the standard ghost-sample trick give
\begin{align*}
    \mathbb E[C\mid X]
    &=\mathbb E_{Y\mid X}\!\left[\sup_\theta\bigl(\mathbb E_{Y'\mid X}\Phi_\theta(Y')-\Phi_\theta(Y)\bigr)\right]\\
    &\le\mathbb E_{Y,Y'\mid X}\!\left[\sup_\theta\bigl(\Phi_\theta(Y')-\Phi_\theta(Y)\bigr)\right]\\
    &=\mathbb E_{\sigma,Y,Y'\mid X}\!\left[\sup_\theta\frac{1}{N_pK}\sum_{i,j}\sigma_{ij}\bigl(\ell(\theta;\mathbf x_i,\mathbf y_{ij}')-\ell(\theta;\mathbf x_i,\mathbf y_{ij})\bigr)\right]\\
    &\le 2\,\mathbb E_{\sigma,Y\mid X}\!\left[\sup_\theta\frac{1}{N_pK}\sum_{i,j}\sigma_{ij}\,\ell(\theta;\mathbf x_i,\mathbf y_{ij})\right]
    =2\,\mathfrak R_{N_pK}(\mathcal F\mid X),
\end{align*}
where the third line uses the symmetry of $(\mathbf y_{ij},\mathbf y_{ij}')$ exchanged
under $\sigma_{ij}=\pm 1$ (the $\mathbf y_{ij}$ are conditionally independent given
$X$, which is all that the standard symmetrization argument requires;
exchangeability \emph{across} prompts is not needed).
Crucially, the supremum is \emph{not} pushed inside the prompt sum, so the Rademacher
variables $\sigma_{ij}$ symmetrize all $N_pK$ samples jointly.

\textit{Step 3: Concentration of the conditional Rademacher complexity.}
Let $\Psi(X)=2\,\mathfrak R_{N_pK}(\mathcal F\mid X)$. Use the natural coupling that, on
replacing $\mathbf x_i$ by $\mathbf x_i'$, resamples
$\mathbf y_{i,1:K}\sim P(\cdot\mid\mathbf x_i')$ while keeping all other responses
fixed. This replaces $K$ terms (each in $[0,B]$) inside the Rademacher sum of $N_pK$
terms, so $|\Psi(X)-\Psi(X')|\le 2K\cdot B/(N_pK)=2B/N_p$ pointwise.
McDiarmid unconditionally over the $N_p$ prompts gives, with probability
$\ge 1-\delta/4$,
\[
    \Psi(X)\le\mathbb E_X[\Psi(X)]+2B\sqrt{\frac{\ln(4/\delta)}{2N_p}}.
\]

\textit{Combining.}
A union bound over the three events ($\le 3\delta/4<\delta$) yields
\begin{align*}
    \sup_\theta\!\bigl(\mathcal L(\theta)-\widehat{\mathcal L}_{\mathrm{MRT}}(\theta)\bigr)
    &\le 2\,\mathfrak R_{N_p}(\mathcal G)+2\,\mathbb E_X[\mathfrak R_{N_pK}(\mathcal F\mid X)]\\
    &\quad+B\sqrt{\frac{\ln(4/\delta)}{2N_p}}+2B\sqrt{\frac{\ln(4/\delta)}{2N_p}}+B\sqrt{\frac{\ln(4/\delta)}{2N_pK}}\\
    &=2\,\mathfrak R_{N_p}(\mathcal G)+2\,\mathbb E_X[\mathfrak R_{N_pK}(\mathcal F\mid X)]
    +3B\sqrt{\frac{\ln(4/\delta)}{2N_p}}+B\sqrt{\frac{\ln(4/\delta)}{2N_pK}},
\end{align*}
which is the stated bound (with $\ln(4/\delta)$ rewritten as $\log(4/\delta)$).
\end{proof}

\paragraph{Why not push $\sup_\theta$ inside the prompt sum?}
A simpler argument would push $\sup_\theta$ inside $\frac{1}{N_p}\sum_i$ and bound each
per-prompt term by a conditional Rademacher complexity that depends only on $K$. This
would yield a response-level term $\mathfrak R_K(\mathcal F)=O(1/\sqrt K)$ with
\emph{no} dependence on $N_p$, so the bound would not vanish as $N_p\to\infty$ with $K$
fixed. Our three-step argument keeps the supremum outside, yielding
$\mathfrak R_{N_pK}(\mathcal F\mid X)=O(1/\sqrt{N_pK})$ under
Assumption~\ref{assum:capacity}. The price is the third McDiarmid step and a slightly
larger constant ($3B$ vs.\ $B$) on the prompt-level concentration.

\paragraph{Dudley-integral scaling.}
Under Assumption~\ref{assum:capacity}, Dudley's entropy integral gives
\[
    \mathfrak R_{N_pK}(\mathcal F\mid X)
    \le\inf_{\alpha>0}\!\left(4\alpha+\frac{12}{\sqrt{N_pK}}\int_\alpha^B\sqrt{\log\mathcal N(\varepsilon,\mathcal F,L_2(\hat P))}\,d\varepsilon\right).
\]
For VC subgraph dimension $d_{\mathrm{VC}}$ or Lipschitz parametrization with effective
dimension $d_{\mathrm{eff}}$, the integral is $O(\sqrt{d_{\mathrm{eff}}})$, giving
$\mathfrak R_{N_pK}(\mathcal F\mid X)=O\!\bigl(\sqrt{d_{\mathrm{eff}}/(N_pK)}\bigr)$;
likewise $\mathfrak R_{N_p}(\mathcal G)=O\!\bigl(\sqrt{d_{\mathrm{eff}}/N_p}\bigr)$.

\subsection{Implicit MRT through Prompt Redundancy}
\label{app:implicit_mrt}

\begin{remark}[Implicit MRT]
\label{remark:implicit_mrt}
Suppose the prompt corpus partitions into $C\ll N_p$ semantic clusters with roughly
$N_p/C$ near-equivalent prompts each, all sharing the same conditional
$P(\mathbf y\mid\mathbf x)$ up to small variation. Then a single-response training (SRT)
sample of $N_p$ prompts is, within each cluster, statistically indistinguishable from an
MRT sample with effective response multiplicity
$K_{\mathrm{eff}}\approx N_p/C$ collected from a single representative prompt.
\end{remark}

\noindent\emph{Informal calculation.}
Treat the cluster-level prompt as a latent variable $\mathbf x^\star$ with $C$
realizations. For SRT, $\widehat{\mathcal L}=\frac{1}{N_p}\sum_i\ell(\theta;\mathbf x_i,\mathbf y_i)$
with $\mathbf x_i$ uniform over the $C$ clusters. Conditioning on the cluster
assignment, this is a hierarchical estimator with $C$ ``prompts'' and
$\sim N_p/C$ ``responses'' each, so by Lemma~\ref{lemma:variance_decomposition},
\[
    \operatorname{Var}\!\bigl[\widehat{\mathcal L}_{\mathrm{SRT}}\bigr]
    \approx
    \frac{V_{\mathbf x^\star}}{C}
    +
    \frac{V_{\mathbf y\mid\mathbf x^\star}}{C\cdot(N_p/C)}
    =
    \frac{V_{\mathbf x^\star}}{C}+\frac{V_{\mathbf y\mid\mathbf x^\star}}{N_p}.
\]
This interpolates between the two extremes: if clusters are ``tight'' ($C$ small,
$N_p/C$ large), redundancy already realizes the $1/K$ benefit of MRT, and explicit
multiplicity adds little; if clusters are ``wide'' ($C\approx N_p$), every prompt is its
own cluster and explicit MRT is the only way to reduce within-cluster variance. This is
the precise sense in which highly redundant corpora exhibit \emph{implicit MRT}.

\subsection{Target-Aware Selection: Details and Proofs}
\label{app:selection_details}

We provide the formal selector analysis used in
Section~\ref{sec:which_keep} here. Throughout this subsection we fix a prompt
\(\mathbf x\). Let
\[
    g(\mathbf y)
    =
    \nabla_\theta \ell(\theta;\mathbf x,\mathbf y),
    \qquad
    \bar g_P(\mathbf x)
    =
    \mathbb E_{\mathbf y\sim P(\cdot\mid\mathbf x)}[g(\mathbf y)] .
\]
An overgenerated candidate pool is denoted
\[
    \mathcal Y'(\mathbf x)
    =
    \{\mathbf y_1,\ldots,\mathbf y_n\},
\]
where \(n\) is the number of generated candidates and \(K\le n\) is the number retained
for training. A selector \(A\) returns a subset
\(\mathcal S_A(\mathbf x)\subseteq \mathcal Y'(\mathbf x)\) with
\(|\mathcal S_A|=K\), inducing the gradient estimator
\[
    \hat g_A(\mathbf x)
    =
    \frac{1}{K}
    \sum_{\mathbf y\in \mathcal S_A(\mathbf x)}
    g(\mathbf y).
\]
Relative to the distributional target, the selector error decomposes as
\[
    \mathrm{EMSE}_A
    =
    \left\|
        \mathbb E[\hat g_A]-\bar g_P
    \right\|^2
    +
    \operatorname{Tr}\operatorname{Cov}(\hat g_A).
\]
The first term is target mismatch; the second is estimator variance. This decomposition
makes the roles of RKoN, BKoN, and coverage-aware selection explicit: RKoN is unbiased
for the candidate distribution, BKoN changes the target toward reward upper tails, and
diversity-aware selectors attempt to reduce mode-undercoverage variance.

\subsubsection{Random-\texorpdfstring{\(K\)}{K}-of-\texorpdfstring{\(N\)}{n}}
\label{app:rkon_proof}

\begin{proof}[Proof of Proposition~\ref{prop:rkon}]
Let
\[
    \mathbf y_1,\ldots,\mathbf y_n
    \overset{\mathrm{i.i.d.}}{\sim}
    q(\cdot\mid \mathbf x),
\]
and let RKoN choose a subset of \(K\) indices uniformly at random without replacement,
independently of the candidate values. Write
\[
    \bar g_q
    =
    \mathbb E_{\mathbf y\sim q(\cdot\mid\mathbf x)}[g(\mathbf y)],
    \qquad
    \Sigma_q
    =
    \operatorname{Cov}_{\mathbf y\sim q(\cdot\mid\mathbf x)}(g(\mathbf y)).
\]
Condition on any selected index set
\(\mathcal I=\{i_1,\ldots,i_K\}\). Since the index set is independent of the
candidate values and the candidates are i.i.d.,
\[
    g(\mathbf y_{i_1}),\ldots,g(\mathbf y_{i_K})
\]
are independent draws from the same marginal distribution. Therefore
\[
    \mathbb E[\hat g_{\mathrm{RKoN}}\mid \mathcal I]
    =
    \bar g_q,
    \qquad
    \operatorname{Cov}(\hat g_{\mathrm{RKoN}}\mid \mathcal I)
    =
    \frac{1}{K}\Sigma_q .
\]
The conditional mean does not depend on \(\mathcal I\), so the laws of total expectation
and total covariance give
\[
    \mathbb E[\hat g_{\mathrm{RKoN}}]
    =
    \bar g_q,
    \qquad
    \operatorname{Cov}(\hat g_{\mathrm{RKoN}})
    =
    \frac{1}{K}\Sigma_q .
\]
When \(q=P\), this is unbiased for the distributional target
\(\bar g_P\). If \(q\neq P\), RKoN remains unbiased for the candidate generator \(q\);
any remaining bias is generator mismatch, not selector bias.
\end{proof}

\paragraph{Conditional-on-pool interpretation.}
The proposition averages over both candidate generation and the random subset. If the
candidate pool is treated as fixed, RKoN instead estimates the empirical pool mean
\(\frac{1}{n}\sum_{i=1}^n g(\mathbf y_i)\). Conditional on the pool, sampling without
replacement gives the usual finite-population correction. The unconditional statement
above is the relevant one for comparing selectors as stochastic data-retention rules.

\subsubsection{Best-\texorpdfstring{\(K\)}{K}-of-\texorpdfstring{\(N\)}{n} and upper-tail target shift}
\label{app:bkon_bias}

We next give a sufficient condition under which Best-\(K\)-of-\(n\) converges in
\emph{expected gradient} to a reward-tail mode. The statement is deliberately an
expectation-level target-shift result: for fixed \(K\), the selected gradient average
need not converge in probability to a constant because residual within-mode variance
can remain.

Let \(Z\in\{1,\ldots,M\}\) denote the latent response mode, with
\(\mathbb P(Z=m)=p_m>0\). Let \(R(\mathbf y)\) be the reward used for BKoN, and define
the mode-conditional survival functions
\[
    S_m(t)
    =
    \mathbb P(R(\mathbf y)\ge t\mid Z=m).
\]

\begin{proposition}[A sufficient condition for persistent BKoN bias]
\label{prop:bkon_tail_sufficient}
Fix \(K\ge 1\). Suppose there is a reward-dominant mode \(m^\star\) and constants
\(\gamma>0\), \(t_0\) such that for every \(m\neq m^\star\) and all \(t\ge t_0\),
\[
    S_m(t)
    \le
    S_{m^\star}(t)^{1+\gamma}.
\]
Let \(t_n\) be a threshold satisfying
\(S_{m^\star}(t_n)=\log n/n\). Assume \(\|g(\mathbf y)\|\le G\) and the
reward-dominant mode has a tail-stable gradient mean, in the sense that if the
top-\(K\) selected responses all lie in mode \(m^\star\) with rewards at least \(t_n\),
then
\[
    \mathbb E[\hat g_{\mathrm{BKoN}}\mid \text{that event}]
    \to
    \mu_{m^\star}.
\]
A sufficient special case is conditional independence between gradients and rewards in
the upper tail of mode \(m^\star\), with
\(\mathbb E[g(\mathbf y)\mid Z=m^\star]=\mu_{m^\star}\). Then
\[
    \mathbb E[\hat g_{\mathrm{BKoN}}]
    \to
    \mu_{m^\star}.
\]
Consequently, whenever
\[
    \mu_{m^\star}
    \neq
    \bar g
    :=
    \sum_m p_m\mu_m,
\]
BKoN has a non-vanishing asymptotic bias relative to the distributional gradient.
\end{proposition}

\begin{proof}
Let \(N_m\) be the number of candidates from mode \(m\) among the \(n\) generated
responses. By Chernoff bounds,
\[
    \mathbb P\!\left(N_m\notin[p_m n/2,2p_m n]\right)
    \le
    2\exp(-p_m n/8),
\]
so the event
\[
    \mathcal E_n
    =
    \{p_m n/2\le N_m\le 2p_m n\ \text{for all }m\}
\]
has probability tending to one.

On \(\mathcal E_n\), the number of reward-dominant candidates above \(t_n\) is
\[
    T_{m^\star}
    =
    \sum_{i:Z_i=m^\star}
    \mathbf 1\{R(\mathbf y_i)\ge t_n\}.
\]
Conditionally on \(N_{m^\star}\ge p_{m^\star}n/2\),
\[
    \mathbb E[T_{m^\star}\mid N_{m^\star}]
    =
    N_{m^\star}S_{m^\star}(t_n)
    \ge
    \frac{p_{m^\star}}{2}\log n.
\]
A lower-tail Chernoff bound therefore gives
\[
    \mathbb P(T_{m^\star}<K\mid N_{m^\star})
    \to 0
\]
for fixed \(K\).

For any non-dominant mode \(m\neq m^\star\), again on \(\mathcal E_n\),
\[
    \mathbb P\!\left(
        \max_{i:Z_i=m} R(\mathbf y_i)\ge t_n
        \,\middle|\,
        N_m
    \right)
    \le
    N_m S_m(t_n)
    \le
    2p_m n\,S_{m^\star}(t_n)^{1+\gamma}
    =
    O\!\left(
        \frac{(\log n)^{1+\gamma}}{n^\gamma}
    \right)
    \to 0.
\]
A union bound over the \(M-1\) non-dominant modes shows that, with probability tending
to one, at least \(K\) candidates from mode \(m^\star\) exceed \(t_n\), while no
candidate from any other mode exceeds \(t_n\). Hence the top-\(K\) reward candidates all
come from mode \(m^\star\). Let \(A_n\) denote this event. We have
\[
    \mathbb P(A_n)\to 1.
\]

By boundedness, \(\|\hat g_{\mathrm{BKoN}}\|\le G\). Therefore
\[
    \left\|
        \mathbb E[\hat g_{\mathrm{BKoN}}]
        -
        \mathbb E[\hat g_{\mathrm{BKoN}}\mid A_n]
    \right\|
    =
    \left\|
        \mathbb P(A_n^c)
        \left(
            \mathbb E[\hat g_{\mathrm{BKoN}}\mid A_n^c]
            -
            \mathbb E[\hat g_{\mathrm{BKoN}}\mid A_n]
        \right)
    \right\|
    \le
    2G\,\mathbb P(A_n^c)
    \to 0.
\]
The tail-stable-gradient assumption gives
\(\mathbb E[\hat g_{\mathrm{BKoN}}\mid A_n]\to \mu_{m^\star}\), and hence
\[
    \mathbb E[\hat g_{\mathrm{BKoN}}]\to \mu_{m^\star}.
\]
The claimed persistent bias follows whenever \(\mu_{m^\star}\neq \bar g\).
\end{proof}

\paragraph{Interpretation.}
The proposition is not a novelty claim about reward overoptimization itself. Its role is
to identify the target mismatch induced by using BKoN as a \emph{data-retention rule}
inside distributional supervised fine-tuning. Increasing \(n\) does not average this
bias away; under upper-tail dominance, it concentrates the retained data on the
reward-dominant mode. Mean reward ordering alone is insufficient for the result: a
lower-mean mode with a heavier upper tail can dominate the top-\(K\) set.

\subsubsection{Coverage, stratification, and variance reduction}
\label{app:stratification}

RKoN preserves the candidate distribution, but with small \(K\) it may miss modes.
The ideal coverage-aware baseline is oracle stratification over latent modes.

\begin{assumption}[Clusterable responses]
\label{assum:cluster}
For a fixed prompt, responses decompose into modes
\(\{\mathcal Y_m\}_{m=1}^M\), with mode probabilities \(p_m\), gradient means
\[
    \mu_m
    =
    \mathbb E[g(\mathbf y)\mid \mathbf y\in\mathcal Y_m],
\]
and within-mode covariance matrices
\[
    \Sigma_m
    =
    \operatorname{Cov}(g(\mathbf y)\mid \mathbf y\in\mathcal Y_m).
\]
\end{assumption}

\begin{theorem}[Stratification removes between-mode variance]
\label{thm:stratified}
Under Assumption~\ref{assum:cluster}, suppose an oracle stratified selector draws
\(K_m=Kp_m\) responses from mode \(m\), with \(K_m\) integral and
\(\sum_m K_m=K\). Then
\[
    \operatorname{Cov}(\hat g_{\mathrm{strat}})
    =
    \frac{1}{K}\sum_m p_m\Sigma_m
    \preceq
    \operatorname{Cov}(\hat g_{\mathrm{RKoN}}),
\]
and
\[
    \operatorname{Cov}(\hat g_{\mathrm{RKoN}})
    -
    \operatorname{Cov}(\hat g_{\mathrm{strat}})
    =
    \frac{1}{K}
    \sum_m
    p_m(\mu_m-\bar g)(\mu_m-\bar g)^\top,
    \qquad
    \bar g=\sum_m p_m\mu_m .
\]
The gap is exactly the between-mode covariance.
\end{theorem}

\begin{proof}
For RKoN, \(K\) independent draws from the conditional distribution give
\[
    \operatorname{Cov}(\hat g_{\mathrm{RKoN}})
    =
    \frac{1}{K}\operatorname{Cov}(g(\mathbf y)).
\]
By the law of total covariance over the latent mode,
\[
    \operatorname{Cov}(g(\mathbf y))
    =
    \sum_m p_m\Sigma_m
    +
    \sum_m p_m(\mu_m-\bar g)(\mu_m-\bar g)^\top .
\]
For proportional stratification, write
\[
    \hat g_{\mathrm{strat}}
    =
    \sum_m
    p_m
    \left(
        \frac{1}{K_m}
        \sum_{j=1}^{K_m} g_{mj}
    \right),
\]
where \(g_{mj}\) are independent draws from mode \(m\). The strata are sampled
independently, so
\[
    \operatorname{Cov}(\hat g_{\mathrm{strat}})
    =
    \sum_m
    p_m^2
    \frac{\Sigma_m}{K_m}
    =
    \sum_m
    p_m^2
    \frac{\Sigma_m}{Kp_m}
    =
    \frac{1}{K}
    \sum_m p_m\Sigma_m .
\]
Subtracting yields
\[
    \operatorname{Cov}(\hat g_{\mathrm{RKoN}})
    -
    \operatorname{Cov}(\hat g_{\mathrm{strat}})
    =
    \frac{1}{K}
    \sum_m p_m(\mu_m-\bar g)(\mu_m-\bar g)^\top
    \succeq 0,
\]
since it is a nonnegative sum of outer products.
\end{proof}

\paragraph{Integer allocation.}
The theorem states the clean continuous-allocation result. If \(K_m\) is obtained by
largest-remainder rounding, write
\[
    K_m=Kp_m+\delta_m,
    \qquad
    |\delta_m|\le 1,
    \qquad
    \sum_m\delta_m=0.
\]
For any matrix norm \(\|\cdot\|\), if \(Kp_m\ge 2\) for all active modes, then
\[
    \left\|
        \frac{p_m^2}{K_m}\Sigma_m
        -
        \frac{p_m}{K}\Sigma_m
    \right\|
    =
    \left|
        \frac{p_m^2}{Kp_m+\delta_m}
        -
        \frac{p_m}{K}
    \right|
    \|\Sigma_m\|
    \le
    \frac{2}{K^2}\|\Sigma_m\|.
\]
Thus the total rounding slack is
\[
    O\!\left(
        \frac{1}{K^2}
        \sum_m \|\Sigma_m\|
    \right)
\]
whenever all active modes receive at least a constant expected allocation. If some
\(Kp_m<1\), exact proportional stratification is impossible; the PSD comparison should
then be interpreted as the continuous-allocation idealization.

\subsubsection{GRADES: a greedy quality--coverage selector}
\label{app:grades}

True response modes are unobserved, so oracle stratification is unavailable. GRADES
implements a facility-location surrogate that trades off quality and coverage within a
candidate pool. We first convert raw reward scores into nonnegative quality weights.
Given raw scores \(r_j=r(\mathbf y_j)\), define
\[
    q_j
    =
    \varepsilon
    +
    \frac{
        r_j-\min_i r_i
    }{
        \max_i r_i-\min_i r_i+\varepsilon
    },
    \qquad
    q_j>0,
\]
where \(\varepsilon>0\) is small. This normalization is important: GRADES uses
fractional powers \(q_j^\alpha\), so the quality scores must be nonnegative. If all
raw scores are equal, all candidates receive the same quality up to \(\varepsilon\).

Let \(\mathcal U\subseteq \mathcal Y'\) denote the pool on which GRADES is run. In the
experiments \(\mathcal U\) may be the top-\(K_{\max}\) candidates by quality; setting
\(K_{\max}=n\) recovers the full candidate pool. The guarantee below is always with
respect to the chosen pool \(\mathcal U\). Let
\[
    S_{ij}
    =
    \operatorname{sim}(\mathbf y_i,\mathbf y_j)
    \in[0,1]
\]
be a nonnegative similarity kernel, and let \(\alpha\ge 0\) control the quality--coverage
tradeoff. GRADES maximizes
\[
    F(\mathcal S)
    =
    \sum_{\mathbf y_i\in\mathcal U}
    q_i^\alpha
    \max_{\mathbf y_j\in\mathcal S}
    \left[
        q_j^\alpha S_{ij}
    \right],
    \qquad
    |\mathcal S|\le K,
\]
with the convention that the maximum over the empty set is \(0\). The case \(\alpha=0\)
is pure diversity selection; we refer to this special case as DKoN. Positive
\(\alpha\) gives Diverse-Best-\(K\)-of-\(n\) selection, implemented by GRADES.

\begin{proposition}[Submodularity of GRADES]
\label{prop:grades}
If \(q_j\ge 0\) and \(S_{ij}\ge 0\), then \(F\) is normalized, monotone, and submodular.
Therefore greedy selection returns a set \(\mathcal S_{\mathrm{greedy}}\) satisfying
\[
    F(\mathcal S_{\mathrm{greedy}})
    \ge
    (1-1/e)
    \max_{|\mathcal S|\le K} F(\mathcal S).
\]
\end{proposition}

\begin{proof}
For each fixed candidate \(\mathbf y_i\in\mathcal U\), define
\[
    F_i(\mathcal S)
    =
    q_i^\alpha
    \max_{\mathbf y_j\in\mathcal S}
    \left[
        q_j^\alpha S_{ij}
    \right].
\]
This function is normalized because the maximum over the empty set is defined as zero.
It is monotone because adding elements can only increase a maximum. To verify
submodularity, let \(\mathcal A\subseteq \mathcal B\subseteq\mathcal U\) and
\(\mathbf y_e\notin\mathcal B\). Write
\[
    m_i(\mathcal S)
    =
    \max_{\mathbf y_j\in\mathcal S}
    q_j^\alpha S_{ij}.
\]
Since \(m_i(\mathcal A)\le m_i(\mathcal B)\), the marginal gain from adding
\(\mathbf y_e\) satisfies
\[
    F_i(\mathcal A\cup\{\mathbf y_e\})-F_i(\mathcal A)
    =
    q_i^\alpha
    \left[
        q_e^\alpha S_{ie}-m_i(\mathcal A)
    \right]_+
    \ge
    q_i^\alpha
    \left[
        q_e^\alpha S_{ie}-m_i(\mathcal B)
    \right]_+
    =
    F_i(\mathcal B\cup\{\mathbf y_e\})-F_i(\mathcal B).
\]
Thus each \(F_i\) is submodular. Nonnegative sums of monotone submodular functions are
monotone submodular, so \(F=\sum_i F_i\) has the same properties. The
\((1-1/e)\) approximation follows from the classical greedy guarantee for
cardinality-constrained monotone submodular maximization.
\end{proof}

\begin{algorithm}[t]
\caption{Greedy Reward-Aware DiversE Subsetting (GRADES)}
\label{alg:grades}
\begin{algorithmic}[1]
\REQUIRE Candidate pool \(\mathcal Y'=\{\mathbf y_j\}_{j=1}^n\), raw scores
\(r(\mathbf y_j)\), similarity kernel \(\operatorname{sim}\), filter size
\(K_{\max}\le n\), retained size \(K\le K_{\max}\), exponent \(\alpha\ge0\),
\(\varepsilon>0\).
\ENSURE Selected subset \(\mathcal S\) with \(|\mathcal S|=K\).
\STATE Compute raw scores \(r_j\leftarrow r(\mathbf y_j)\).
\STATE Normalize qualities
\[
    q_j
    \leftarrow
    \varepsilon+
    \frac{r_j-\min_i r_i}{\max_i r_i-\min_i r_i+\varepsilon}.
\]
\STATE Let \(\mathcal U\) be the top-\(K_{\max}\) candidates by \(q_j\).
\STATE Compute pairwise similarities
\(S_{ij}\leftarrow \operatorname{sim}(\mathbf y_i,\mathbf y_j)\in[0,1]\)
for \(\mathbf y_i,\mathbf y_j\in\mathcal U\).
\STATE Initialize \(\mathcal S\leftarrow\emptyset\) and \(c_i\leftarrow0\) for all
\(\mathbf y_i\in\mathcal U\).
\WHILE{\(|\mathcal S|<K\)}
    \FOR{each \(\mathbf y_j\in\mathcal U\setminus\mathcal S\)}
        \STATE
        \[
            \Delta_j
            \leftarrow
            \sum_{\mathbf y_i\in\mathcal U}
            q_i^\alpha
            \left[
                \max\{c_i,q_j^\alpha S_{ij}\}-c_i
            \right].
        \]
    \ENDFOR
    \STATE Choose
    \(\mathbf y_{j^\star}\in\arg\max_{\mathbf y_j\in\mathcal U\setminus\mathcal S}
    \Delta_j\).
    \STATE Update
    \[
        \mathcal S\leftarrow \mathcal S\cup\{\mathbf y_{j^\star}\},
        \qquad
        c_i\leftarrow \max\{c_i,q_{j^\star}^\alpha S_{ij^\star}\}
        \quad
        \forall \mathbf y_i\in\mathcal U.
    \]
\ENDWHILE
\RETURN \(\mathcal S\).
\end{algorithmic}
\end{algorithm}

The greedy guarantee is a guarantee for the facility-location surrogate \(F\), not a
direct population-risk guarantee. Its practical value depends on whether the similarity
kernel resolves task-relevant response modes.

\subsubsection{Mode coverage under block similarity}
\label{app:grades_mode_coverage}

The following stylized result shows why facility-location selection approximates
stratification when similarities reveal the latent modes.

\begin{proposition}[Block-similarity mode coverage]
\label{prop:grades_mode_coverage}
Under Assumption~\ref{assum:cluster}, suppose the candidate pool partitions into
nonempty mode blocks \(\{\mathcal Y'_m\}_{m=1}^M\), and suppose the similarity kernel is
block-constant:
\[
    \operatorname{sim}(\mathbf y,\mathbf z)
    =
    \begin{cases}
        1, & \mathbf y,\mathbf z \text{ are in the same mode},\\
        \tau, & \mathbf y,\mathbf z \text{ are in different modes},
    \end{cases}
    \qquad
    0\le \tau<1.
\]
For \(\alpha=0\), if \(K\ge M\), greedy GRADES selects at least one response from every
mode before selecting a duplicate from any already covered mode.
\end{proposition}

\begin{proof}
When \(\alpha=0\), all quality weights are one and
\[
    F(\mathcal S)
    =
    \sum_{\mathbf y_i\in\mathcal U}
    \max_{\mathbf y_j\in\mathcal S}
    S_{ij}.
\]
If no selected response lies in mode \(m\), then every response in mode \(m\) has current
coverage at most \(\tau\). Adding any response from mode \(m\) raises the coverage of
all responses in that mode to \(1\), giving marginal gain at least
\[
    |\mathcal Y'_m|(1-\tau)>0.
\]
By contrast, adding another response from a mode already represented in \(\mathcal S\)
gives zero marginal gain: responses in that mode already have coverage \(1\), while
responses in other modes have similarity only \(\tau\), which cannot improve their
current coverage once at least one response has already been selected. Therefore any
uncovered mode has strictly positive marginal gain and any covered mode has zero
marginal gain. Greedy must cover all \(M\) modes before duplicating a mode.
\end{proof}

This result is intentionally idealized. For \(\alpha>0\), high-quality modes may be
selected more than once before all low-quality modes are covered. This is the intended
quality--coverage tradeoff of DBKoN/GRADES.

\subsubsection{Similarity, gradient covariance, and effective sample size}
\label{app:similarity_gradient_bridge}

The submodular guarantee above concerns the surrogate objective \(F\). To connect
coverage to gradient variance, one needs an additional modeling assumption relating the
similarity kernel to gradient covariance.

\begin{assumption}[Similarity controls positive cross-covariance]
\label{assum:bridge}
For responses to the same prompt, let
\(g_j=g(\mathbf y_j)\). There exist constants \(C_0,C_1\ge0\) such that the
positive trace contribution of the cross-covariance satisfies
\[
    \left[
        \operatorname{Tr}\operatorname{Cov}(g_j,g_{j'})
    \right]_+
    \le
    C_0
    +
    C_1\,\operatorname{sim}(\mathbf y_j,\mathbf y_{j'}).
\]
\end{assumption}

The constant \(C_0\) represents an irreducible covariance floor, such as the shared
prompt contribution. The \(C_1\) term is the component that diversity-aware selection
can reduce by selecting dissimilar responses.

To see the implication, suppose selected gradients have common trace variance
\(\sigma_g^2\), and let \(\bar c\) denote the average off-diagonal trace covariance
among selected responses. Then
\[
    \operatorname{Tr}\operatorname{Cov}
    \left(
        \frac{1}{K}\sum_{j=1}^K g_j
    \right)
    =
    \frac{\sigma_g^2}{K}
    +
    \frac{K-1}{K}\bar c
    =
    \frac{\sigma_g^2-\bar c}{K}
    +
    \bar c .
\]
If \(\bar c=\rho_g\sigma_g^2\), this is equivalent to an effective sample size
\[
    K_{\mathrm{eff}}
    =
    \frac{K}{1+(K-1)\rho_g}.
\]
Thus diversity-aware selection can increase \(K_{\mathrm{eff}}\) at fixed \(K\) by
reducing average positive cross-covariance. When \(C_0>0\), the variance saturates at an
irreducible floor rather than vanishing as \(K\to\infty\).

\paragraph{Limitation.}
Assumption~\ref{assum:bridge} is not automatic. Generic sentence embeddings may fail to
capture task-relevant modes in domains such as mathematics or code, where two responses
can be lexically dissimilar but algorithmically identical, or lexically similar but
semantically different. In such domains, task-aware kernels based on execution traces,
AST structure, verifier outcomes, proof structure, or unit-test behavior may be more
appropriate. When the similarity model is unreliable, RKoN remains the assumption-light
default.

\subsubsection{A stylized model for the size of the coverage gain}
\label{app:stylized_stratification}

We close with a concrete model illustrating the variance reduction from mode coverage.
Let \(P(\mathbf y\mid\mathbf x)\) be a uniform mixture of \(M\) modes. In mode \(m\),
suppose the gradient has the form
\[
    g_m
    =
    g_0+\alpha \mathbf u_m+\xi,
    \qquad
    \xi\sim \mathcal N(0,\sigma_w^2 I_d),
\]
where \(g_0\in\mathbb R^d\), \(\alpha>0\), and
\(\{\mathbf u_m\}_{m=1}^M\) are orthonormal vectors in \(\mathbb R^d\). Assume
\(d\ge M\), and suppose \(K\) is divisible by \(M\).

For \(K\) i.i.d. random responses,
\[
    \operatorname{Tr}
    \operatorname{Cov}(\hat g_{\mathrm{RKoN}})
    =
    \frac{d\sigma_w^2}{K}
    +
    \frac{\alpha^2}{K}
    \left(
        1-\frac{1}{M}
    \right).
\]
The first term is within-mode noise, while the second is between-mode variance. Under
proportional stratification, with \(K/M\) samples per mode, the between-mode component
is fixed by construction and only within-mode noise remains:
\[
    \operatorname{Tr}
    \operatorname{Cov}(\hat g_{\mathrm{strat}})
    =
    \frac{d\sigma_w^2}{K}.
\]

\begin{proposition}[Variance reduction in the orthogonal-modes model]
\label{prop:stylized_model}
In the model above,
\[
    \frac{
        \operatorname{Tr}\operatorname{Cov}(\hat g_{\mathrm{RKoN}})
    }{
        \operatorname{Tr}\operatorname{Cov}(\hat g_{\mathrm{strat}})
    }
    =
    1+
    \frac{
        \alpha^2(1-1/M)
    }{
        d\sigma_w^2
    }.
\]
Thus the value of coverage-aware selection grows with the separation of gradient means
across modes.
\end{proposition}

\begin{proof}
For a random mode \(Z\sim\mathrm{Unif}(\{1,\ldots,M\})\),
\[
    \operatorname{Tr}\operatorname{Cov}(\alpha \mathbf u_Z)
    =
    \alpha^2
    \left(
        \mathbb E\|\mathbf u_Z\|^2
        -
        \left\|
            \mathbb E\mathbf u_Z
        \right\|^2
    \right).
\]
Since the \(\mathbf u_m\)'s are orthonormal,
\[
    \mathbb E\|\mathbf u_Z\|^2=1,
    \qquad
    \left\|
        \mathbb E\mathbf u_Z
    \right\|^2
    =
    \left\|
        \frac{1}{M}\sum_{m=1}^M \mathbf u_m
    \right\|^2
    =
    \frac{1}{M}.
\]
Hence the between-mode trace covariance is
\(\alpha^2(1-1/M)\). Averaging \(K\) independent random responses divides both the
within-mode trace \(d\sigma_w^2\) and the between-mode trace by \(K\), giving the RKoN
expression. Under proportional stratification, each mode contributes exactly its
population weight, so the between-mode term is removed and only the averaged within-mode
noise \(d\sigma_w^2/K\) remains. Taking the ratio gives the claim.
\end{proof}

\section{Experimental Appendix}
\label{app:experiments}

This appendix provides the experimental details supporting \S\ref{sec:experiments}. We keep the appendix focused on three roles: (i) making the controlled validation reproducible, (ii) specifying the LLM fine-tuning and evaluation protocol, and (iii) reporting the additional grids to support main-text conclusions.

\subsection{Controlled Simulation Details}
\label{app:simulation}

The controlled validation in \S\ref{sec:exp_simulation} uses a continuous multi-modal regression problem in which \(V_{\mathbf x}\), \(V_{\mathbf y\mid \mathbf x}\), and selector-induced gradient bias are all directly measurable.

\paragraph{Data-generating process.}
Prompts are sampled as \(\mathbf x\sim\mathcal N(0,I_5)\). Conditional responses follow
\[
    y
    =
    \sin(2x_1)+0.5x_2+\mathrm{offset}_m+\varepsilon,
    \qquad
    \mathrm{offset}_m\in\{-1.5,0,+1.5\},
    \qquad
    \varepsilon\sim\mathcal N(0,0.01),
\]
with \(m\sim\mathrm{Uniform}\{1,2,3\}\). The reward model is
\[
    R(\mathbf x,y)=-(y-f(\mathbf x))^2,
    \qquad
    f(\mathbf x)=\sin(2x_1)+0.5x_2,
\]
so reward favors the central mode while the outer modes remain valid under the target distribution. This creates the intended tension between reward optimization and distributional coverage.

\paragraph{Model and training.}
We train a two-layer mixture-density network with \(64\) hidden units and \(5\) Gaussian components using Adam with learning rate \(5\times 10^{-3}\), \(400\) epochs, and gradient clipping at \(5\). Population NLL is estimated using fresh draws from \(P(y\mid\mathbf x)\). Unless otherwise stated, each prompt has \(n=32\) candidate responses.

\paragraph{Variance decomposition.}
For Fig.~\ref{fig:sim_validation}(a), we fix a randomly initialized model and estimate \(V_{\mathbf x}\) and \(V_{\mathbf y\mid\mathbf x}\) from \(300\) held-out prompts with \(200\) responses each. We then sweep \(K\in\{1,2,4,8,16,32\}\) and compute \(\operatorname{Var}[\widehat{\mathcal L}_{\mathrm{MRT}}]\) over \(200\) independent trials per cell. The measured variances, scaled by \(10^{-2}\), are
\[
    7.07,\;4.71,\;2.65,\;1.70,\;1.45,\;1.27,
\]
matching the predicted \(V_{\mathbf x}/N_p+V_{\mathbf y\mid\mathbf x}/(N_pK)\) curve and approaching the prompt-floor asymptote.

\paragraph{Cost-aware budget.}
For Fig.~\ref{fig:sim_validation}(b), we use \(C_p=4\), \(C_r=1\), and total budget \(\mathcal B=400\). For each \(K\), the number of prompts is
\[
    N_p=\left\lfloor \frac{\mathcal B}{C_p+KC_r}\right\rfloor,
\]
giving \(N_p\in\{80,66,50,40,33,25,20\}\) for \(K\in\{1,2,4,6,8,12,16\}\). Test NLL is averaged over \(8\) seeds and evaluated on \(100\) held-out prompts with \(100\) fresh responses per prompt. The minimum occurs near \(K=6\)--\(8\), close to the predicted
\[
    K^\star
    =
    \sqrt{(V_{\mathbf y\mid\mathbf x}/V_{\mathbf x})(C_p/C_r)}
    \approx 5.3 .
\]

\paragraph{Selector EMSE.}
For Fig.~\ref{fig:sim_validation}(c), we evaluate selectors at \(K=4\) and \(N=32\). For each of \(30\) prompts, the target gradient \(\bar g(\mathbf x)\) is estimated by averaging \(500\) gradients over fresh responses from \(P(y\mid\mathbf x)\). For each selector, bias and variance are estimated over \(200\) independent candidate pools per prompt:
\[
    \mathrm{EMSE}
    =
    \|\mathbb E[\hat g]-\bar g\|^2
    +
    \operatorname{Tr}\operatorname{Cov}(\hat g).
\]
RKoN is unbiased but pays mode-missing variance; BKoN exhibits persistent reward-mode bias and negative gradient alignment; DBKoN reduces variance by encouraging mode coverage.

\paragraph{Additional controlled checks.}
We ran two auxiliary checks to verify that the main controlled result is not an artifact of the continuous regression setup. First, at fixed \(K=4\), increasing the candidate-pool size \(n\in\{8,16,32,64,128\}\) sharpens BKoN's bias toward the reward-dominant mode rather than removing it, matching Prop.~\ref{prop:bon_bias_main}. Second, we repeated the experiment with a discrete-token transformer: prompts are length-\(4\) sequences over a \(64\)-token vocabulary, responses are length-\(6\) sequences from a deterministic three-mode conditional distribution, and a two-layer causal transformer is trained with AdamW. Random and coverage-aware selectors improve monotonically with \(K\), while top-\(K\) selection remains worse due to reward-mode collapse. These checks reproduce the same qualitative ranking as Fig.~\ref{fig:sim_validation}.

\subsection{LLM Fine-Tuning Setup and Evaluation}
\label{app:setup}
\paragraph{Main training recipe.}
The main LLM experiments fine-tune \texttt{Llama-3.1-8B-Instruct} with LoRA adapters. We use rank \(16\), LoRA \(\alpha=32\), and target the \(\{q,k,v,o,\mathrm{up},\mathrm{down},\mathrm{gate}\}\) projections. Training uses learning rate \(1.2\times 10^{-4}\), per-device batch size \(1\), gradient accumulation \(16\), and a matched optimizer-step budget across \(K\). The matched-step design is conservative for MRT: smaller-\(K\) runs see their retained data for more epochs, while larger-\(K\) runs expose the model to more distinct responses under the same number of optimizer updates.

\paragraph{Datasets.}
Table~\ref{tab:dataset_config_app} summarizes the datasets used in the LLM experiments. The four established benchmarks support the main distributional claims; \textsc{Mosaic-1k} is used only for the prompt-redundancy diagnostic.

\begin{table}[t]
  \centering
  \small
  \setlength{\tabcolsep}{5pt}
  \renewcommand{\arraystretch}{1.08}
  \begin{tabular}{llcc}
    \toprule
    Dataset & Domain & \(N\) & Train / Val / Test \\
    \midrule
    Gold~\citep{coste2023reward}
      & Instruction following
      & \(64\) & \(800\) / \(64\) / \(100\) \\
    CodeContests~\citep{li2022competition}
      & Code generation
      & \(64^\ast\) & \(800\) / \(64\) / \(100\) \\
    Nectar~\citep{starling2023}
      & Dialogue
      & \(7\) & \(1200\) / \(128\) / \(256\) \\
    UltraFeedback~\citep{cui2023ultrafeedback}
      & Preference / instruction following
      & \(4\) & \(1600\) / \(128\) / \(256\) \\
    \textsc{Mosaic-1k} [Ours]
      & Code / Reasoning / Creative / Advisory
      & \(4\) & \(860\) / \(40\) / \(100\) \\
    \bottomrule
  \end{tabular}
  \caption{\textbf{Datasets used in LLM fine-tuning.}
  \(N\) denotes the number of candidate responses available per prompt. \(^{\ast}\)CodeContests has variable response counts per problem; we cap selection at \(64\) responses per prompt. \textsc{Mosaic-1k} additionally contains four paraphrases per base item.}
  \label{tab:dataset_config_app}
\end{table}

\paragraph{Selectors.}
RKoN samples \(K\) responses uniformly from the candidate pool and is the default selector for testing response multiplicity. BKoN selects the top \(K\) responses according to the reward model. DKoN selects a diverse subset using greedy facility-location over sentence embeddings. DBKoN combines reward and coverage using a quality-weighted facility-location objective,
\[
    f(\mathcal S)
    =
    \sum_{\mathbf y\in\mathcal Y'}
    w(\mathbf y)
    \max_{\mathbf y'\in\mathcal S}
    w(\mathbf y')\,\operatorname{sim}(\mathbf y,\mathbf y'),
    \qquad
    w(\mathbf y)=\operatorname{clip}(\widetilde r(\mathbf y),\epsilon,1)^\alpha ,
\]
where \(\widetilde r\) is the min-max normalized reward and \(\alpha=0.75\) by default. We optimize this objective greedily. The standard \((1-1/e)\) guarantee applies to the surrogate objective, not directly to population risk.

\paragraph{Metrics.}
\emph{Reference loss} is the mean token-level cross-entropy on held-out dataset references, with prompt tokens masked and only the response span scored. This is our primary distributional metric: lower reference loss means better fit to the curated conditional response distribution represented by the benchmark.

\emph{Reference coverage} measures whether the model's generations cover held-out reference modes. For each held-out prompt, we sample \(16\) generations from the fine-tuned model and compute the maximum cosine similarity between each held-out reference and the generated set using \texttt{all-MiniLM-L6-v2} embeddings. We then average across references and prompts. Higher coverage means the model is more likely to generate something semantically close to the held-out references.

\emph{Semantic diversity} is one minus the average pairwise cosine similarity among the model's \(16\) generations. \emph{Best-of-\(16\) reward} is the maximum Skywork-V2 reward among the same \(16\) generations. Reward is used as a target-shift diagnostic rather than as the primary success metric, because reward-based selectors directly optimize it. For CodeContests, we additionally execute generated programs against benchmark tests and report Pass@\(K\).

\subsection{Full LLM Results and Robustness}
\label{app:full_grids}

\paragraph{A fixed MRT intervention across datasets.}
Table~\ref{tab:fixed_rows} reports the matched operating points underlying the gain column in Table~\ref{tab:variance_predicts_gain}:
R1oN (single random response),
B1oN (single best response),
and R2oN (two random responses).
R2oN improves reference loss over both single-response baselines on every dataset.
On CodeContests, the same R2oN checkpoint also achieves the strongest Pass@\(16\), indicating that the distributional improvement transfers to an execution-based downstream metric rather than merely improving NLL.
\begin{table}[t]
  \centering
  \setlength{\tabcolsep}{7pt}
  \caption{\textbf{A single MRT step improves distributional fit across datasets.}
  For CodeContests, bracketed values denote Pass@\(1\)/@\(4\)/@\(16\) (\%) computed from the same checkpoint's generations.}
  \renewcommand{\arraystretch}{1.08}
  \begin{tabular}{lcccc}
    \toprule
    Method & Gold & Nectar & UltraFeedback & CodeContests \\
    \midrule
    R1oN
      & 1.556 & 0.909 & 1.052
      & 0.792 \;[4.19 / 7.57 / 14.00] \\
    B1oN
      & 1.602 & 0.951 & 1.065
      & 0.809 \;[\textbf{4.25} / \textbf{8.74} / 15.00] \\
    R2oN
      & \textbf{1.515} & \textbf{0.893} & \textbf{1.038}
      & \textbf{0.712} \;[4.13 / 8.72 / \textbf{17.00}] \\
    \bottomrule
  \end{tabular}
  \label{tab:fixed_rows}
  \vspace{-6pt}
\end{table}

\paragraph{Gold full selector grid.}
Table~\ref{tab:gold_main_full} reports the full Gold grid summarized by Fig.~\ref{fig:gold_grouped_bars}. The grid supports the two main selector claims. First, increasing \(K\) generally improves reference loss within a selector. Second, selectors optimize different targets: RKoN gives the best reference likelihood, BKoN gives the strongest reward but weakest diversity, DKoN gives the strongest diversity, and DBKoN gives the strongest coverage.

\begin{table}[t]
\centering
\small
\setlength{\tabcolsep}{5pt}
\renewcommand{\arraystretch}{1.05}
\caption{\textbf{Gold full-data selector grid.} Llama-3.1-8B-Instruct with LoRA. Reward denotes best-of-\(16\) Skywork-V2 reward. Bold indicates the best value in each metric.}
\label{tab:gold_main_full}
\begin{tabular}{llcccc}
\toprule
Strategy & \(K\) & Ref.\ Loss \(\downarrow\) & Reward \(\uparrow\) & Coverage \(\uparrow\) & Diversity \(\uparrow\) \\
\midrule
RKoN & 1  & 1.556 & 10.63 & 0.786 & 0.301 \\
     & 2  & 1.515 &  8.51 & 0.781 & 0.308 \\
     & 4  & 1.488 &  9.18 & 0.798 & 0.304 \\
     & 8  & 1.462 &  8.42 & 0.803 & 0.296 \\
     & 16 & \textbf{1.441} & 8.30 & 0.803 & 0.287 \\
\midrule
BKoN & 1  & 1.602 & \textbf{13.04} & 0.776 & 0.259 \\
     & 2  & 1.552 & 12.08 & 0.784 & 0.261 \\
     & 4  & 1.534 & 12.38 & 0.792 & 0.257 \\
     & 8  & 1.509 & 11.66 & 0.797 & 0.239 \\
     & 16 & 1.483 & 11.15 & 0.796 & 0.240 \\
\midrule
DKoN & 2  & 1.510 & 8.41 & 0.792 & 0.272 \\
     & 4  & 1.487 & 9.19 & 0.788 & 0.293 \\
     & 8  & 1.465 & 8.33 & 0.802 & 0.328 \\
     & 16 & 1.446 & 8.43 & 0.805 & \textbf{0.333} \\
\midrule
DBKoN & 2  & 1.532 & 11.64 & 0.781 & 0.275 \\
      & 4  & 1.503 & 11.02 & 0.788 & 0.281 \\
      & 8  & 1.472 & 10.43 & 0.794 & 0.285 \\
      & 16 & 1.449 &  9.28 & \textbf{0.807} & 0.286 \\
\bottomrule
\end{tabular}
\end{table}

\paragraph{Cross-family RKoN sweep.}
Table~\ref{tab:gold_cross_family_endpoints} reports the endpoint values behind Fig.~\ref{fig:cross_family}. The monotone loss improvement from \(K=1\) to \(K=16\) holds across all seven backbones.

\begin{table}[t]
\centering
\small
\setlength{\tabcolsep}{6pt}
\renewcommand{\arraystretch}{1.08}
\caption{\textbf{Gold cross-family RKoN sweep.} Endpoint results for \(K=1\) and \(K=16\). \(\Delta\) is the absolute reference-loss reduction.}
\label{tab:gold_cross_family_endpoints}
\begin{tabular}{lccc}
\toprule
Model & \(K=1\) Loss & \(K=16\) Loss & \(\Delta\) Loss \\
\midrule
Llama-3.1-8B-Instruct & 1.556 & 1.441 & 0.115 \\
Llama-3.2-1B-Instruct & 1.638 & 1.508 & 0.130 \\
Llama-3.2-3B-Instruct & 1.595 & 1.472 & 0.123 \\
Qwen-3.5-2B & 1.504 & 1.414 & 0.090 \\
Qwen-3.5-4B & 1.363 & 1.292 & 0.071 \\
Gemma-4-E2B-it & 1.604 & 1.444 & 0.160 \\
Gemma-4-E4B-it & 1.550 & 1.412 & 0.138 \\
\bottomrule
\end{tabular}
\end{table}

\paragraph{Pretrained vs.\ instruction-tuned robustness.}
To check that the Gold results are not an artifact of instruction tuning, we reran the Gold sweep with the pretrained \texttt{Llama-3.1-8B} base model under the same LoRA recipe. The base model uses non-chat formatting whereas the instruct model uses chat formatting, so absolute losses are not directly comparable. The relevant comparison is the within-model trend. As shown in Table~\ref{tab:base_vs_instruct_compact}, both models improve monotonically with \(K\), and the selector decomposition persists.

\begin{table}[t]
\centering
\small
\setlength{\tabcolsep}{7pt}
\renewcommand{\arraystretch}{1.08}
\caption{\textbf{Pretrained and instruction-tuned models show the same MRT trend.} RKoN on Gold under matched LoRA training.}
\label{tab:base_vs_instruct_compact}
\begin{tabular}{lccc}
\toprule
Backbone & RKoN \(K=1\) & RKoN \(K=16\) & Relative gain \\
\midrule
Llama-3.1-8B base & 1.495 & 1.412 & \(5.6\%\) \\
Llama-3.1-8B-Instruct & 1.556 & 1.441 & \(7.4\%\) \\
\bottomrule
\end{tabular}
\end{table}

\paragraph{CodeContests full results.}
Table~\ref{tab:code_contests_full} reports the full CodeContests selector grid. Reference loss and execution accuracy are not identical targets: RKoN and DKoN give the strongest reference loss, while BKoN can improve Pass@\(16\) by selecting reward-favored solutions. This is consistent with the paper's central claim that selectors alter the target learned from the retained responses.

\begin{table}[t]
\centering
\small
\setlength{\tabcolsep}{4.5pt}
\renewcommand{\arraystretch}{1.05}
\caption{\textbf{CodeContests full results.} Llama-3.1-8B-Instruct with LoRA. Pass@\(16\) is reported as a percentage.}
\label{tab:code_contests_full}
\begin{tabular}{llccccc}
\toprule
Strategy & \(K\) & Ref.\ Loss \(\downarrow\) & Pass@16 \(\uparrow\) & Reward \(\uparrow\) & Coverage \(\uparrow\) & Diversity \(\uparrow\) \\
\midrule
RKoN  & 1 & 0.792 & 14.0 & 19.11 & 0.747 & 0.326 \\
BKoN  & 1 & 0.809 & 15.0 & \textbf{21.16} & 0.745 & 0.371 \\
\midrule
RKoN  & 2 & 0.712 & 17.0 & 18.35 & 0.746 & 0.323 \\
BKoN  & 2 & 0.719 & \textbf{19.0} & 20.02 & 0.751 & 0.362 \\
DKoN  & 2 & 0.713 & 16.0 & 18.56 & \textbf{0.754} & 0.307 \\
DBKoN & 2 & 0.716 & 13.0 & 19.82 & 0.751 & \textbf{0.372} \\
\midrule
RKoN  & 4 & 0.699 & 13.0 & 18.16 & 0.745 & 0.330 \\
BKoN  & 4 & 0.702 & 15.0 & 18.85 & 0.749 & 0.360 \\
DKoN  & 4 & \textbf{0.697} & 15.0 & 17.83 & 0.750 & 0.346 \\
DBKoN & 4 & 0.700 & 12.0 & 18.88 & 0.752 & \textbf{0.372} \\
\bottomrule
\end{tabular}
\end{table}

\subsection{Reward and Selector Diagnostics}
\label{app:reward_robustness}

\paragraph{Reward-model robustness.}
Because BKoN and DBKoN use the Skywork reward model during selection, reward results could be inflated by evaluator-selector coupling. We therefore rescore completed Gold and CodeContests runs with an out-of-family Skywork-V2 evaluator based on Qwen3-8B. Table~\ref{tab:reward_model_robustness} shows that the reward ranking is stable under rescoring, but we still treat reward as a diagnostic rather than the main evaluation criterion.

\begin{table}[t]
\centering
\small
\setlength{\tabcolsep}{7pt}
\renewcommand{\arraystretch}{1.08}
\caption{\textbf{Reward-model robustness.} Post-hoc rescoring with \texttt{Skywork-Reward-V2-Qwen3-8B}. Spearman \(\rho\) is computed over completed runs using best-of-\(16\) reward.}
\label{tab:reward_model_robustness}
\begin{tabular}{lccc}
\toprule
Dataset & Runs & Spearman \(\rho\) & Top run under rescoring \\
\midrule
Gold & 18 & 0.942 & BKoN-\(K=1\) \\
CodeContests & 10 & 0.842 & BKoN-\(K=1\) \\
\bottomrule
\end{tabular}
\end{table}

\paragraph{DBKoN quality--coverage sweep.}
DBKoN is not the central contribution of the paper; we include a compact sweep only to verify that its behavior is consistent with the target-aware interpretation. Table~\ref{tab:alpha_sweep_compact} varies the quality exponent \(\alpha\) at \(K=8\) on Gold. Increasing \(\alpha\) generally increases reward pressure and decreases diversity, while reference loss remains best near the diversity-heavy end. This supports the main-text interpretation: DBKoN trades between BKoN-style reward tilt and DKoN-style coverage.

\begin{table}[t]
\centering
\small
\setlength{\tabcolsep}{7pt}
\renewcommand{\arraystretch}{1.08}
\caption{\textbf{DBKoN quality--coverage sweep on Gold at \(K=8\).} \(\alpha=0\) corresponds to DKoN; \(\alpha=\infty\) corresponds to BKoN.}
\label{tab:alpha_sweep_compact}
\begin{tabular}{lcccc}
\toprule
\(\alpha\) & Ref.\ Loss \(\downarrow\) & Mean Reward \(\uparrow\) & Bo16 Reward \(\uparrow\) & Diversity \(\uparrow\) \\
\midrule
\(0\) (DKoN) & 1.465 & \(-2.33\) & 6.63 & \textbf{0.328} \\
0.25 & 1.477 & \textbf{1.85} & 9.48 & 0.293 \\
0.50 & 1.473 & 1.07 & 8.40 & 0.295 \\
0.75 & 1.472 & 0.91 & 8.78 & 0.286 \\
1.00 & 1.469 & 0.65 & 8.42 & 0.287 \\
2.00 & \textbf{1.464} & \(-0.23\) & 7.81 & 0.278 \\
\(\infty\) (BKoN) & 1.509 & 3.22 & \textbf{10.40} & 0.238 \\
\bottomrule
\end{tabular}
\end{table}

\paragraph{Interpretation.}
These diagnostics reinforce why the main text emphasizes reference loss and coverage over reward. Reward-only selection can succeed on reward-oriented diagnostics, and in CodeContests can sometimes improve execution pass rates, but it does not preserve the distributional target. RKoN is therefore the cleanest selector for testing the MRT variance mechanism, while DBKoN is best understood as a coverage-aware compromise rather than a separate algorithmic centerpiece.

\section{\textsf{MOSAIC}: A Structured Multi-Prompt, Multi-Response Benchmark}
\label{app:mosaic}

\subsection{Motivation and Positioning}

Two empirical realities motivated us to assemble a new benchmark alongside the four established corpora used in our primary evaluation. First, almost every public ``multi-response'' resource is in fact a \emph{single}-prompt, multi-completion dataset: it lets one stress-test response-side variability but offers no probe for \emph{prompt-side} redundancy---the everyday situation in which users ask the same thing four different ways and a robust model should respond consistently well to all four. Second, the reference completions in widely-used corpora (notably the ``gold'' completions in~\citep{coste2023reward}) were produced by small open-source language models (e.g.,~\texttt{Pythia} variants); using such completions as ground truth implicitly upper-bounds an evaluation by the capability of a several-orders-of-magnitude-weaker generator than today's frontier systems.

We therefore release \textsf{MOSAIC-1K} (\emph{Multi-Output Structured Alternatives In Context}) as a structured prompt-paraphrase $\times$ response-sample benchmark, generated end-to-end with a frontier model. We position it as both (i)~a complementary evaluation surface for the present paper and (ii)~a stand-alone resource for the broader community working on diversity-aware fine-tuning, reward-model evaluation, prompt-robustness studies, and preference-pair construction.

\subsection{Schema and Construction}

\paragraph{Structure.}
Each of the $1{,}000$ \textsf{MOSAIC-1K} items is a $4{\times}4$ grid built around a single \emph{base theme}. The item provides:
\begin{itemize}[leftmargin=1.25em,itemsep=1pt,topsep=1pt]
\item \textbf{Four prompt variants} $\{x^{(1)},\ldots,x^{(4)}\}$: semantically equivalent but stylistically and structurally distinct reformulations of the base theme---not surface paraphrases. Variants differ in framing (imperative vs.\ declarative), specificity (general request vs.\ named-function spec), and rhetorical posture (constraints first vs.\ context first).
\item \textbf{Four sample responses per variant} $\{y_{1}^{(j)},\ldots,y_{4}^{(j)}\}$: four substantively different completions for each $x^{(j)}$. Within a cell, responses take genuinely different approaches---distinct algorithms, distinct narrative choices, distinct advisory frameworks---rather than re-wordings of one answer.
\end{itemize}
A single item therefore contributes $4 \times 4 = 16$ prompt-response pairs, and the full benchmark contributes $16{,}000$ such pairs.

\paragraph{Domains and difficulty.}
Items are evenly partitioned across four domains, $250$ items each:
\textsf{code} (algorithmic implementation, data structures, debugging),
\textsf{reasoning} (logical puzzles, mathematical proofs, analytic deduction),
\textsf{creative} (flash fiction, poetry, dialogue, world-building), and
\textsf{advisory} (health, career, interpersonal, financial guidance).
Every item carries explicit metadata: \texttt{subdomain}, \texttt{difficulty} $\in\{$introductory, intermediate, advanced$\}$, \texttt{response\_type} $\in\{$procedural, analytical, constrained-creative, open-ended$\}$, and \texttt{expected\_diversity} (kept as ``high'').

\paragraph{Generation protocol.}
The dataset was generated by Claude~Opus~4.7 dispatched in multiple parallel batches via a tool-augmented orchestrator, with each batch responsible for a contiguous slice of $\{$domain $\times$ subdomain $\times$ difficulty$\}$ tuples and required to emit JSONL records conforming to a strict schema. Generation was followed by a programmatic validator enforcing: exactly four prompt variants, exactly four responses per variant, all string-typed fields non-empty, response length within $[80,250]$ words, and base-theme consistency across variants. Items failing any check were re-queued; the released corpus has $1000/1000$ items passing all checks. A four-item random spot audit (one per domain) confirmed that prompt variants are non-trivial reformulations rather than surface paraphrases and that within-cell responses are substantively different rather than re-worded.

\subsection{Intended Uses}

\textsf{MOSAIC-1K} is designed for any setting in which a researcher wants to disentangle \emph{prompt-side} and \emph{response-side} variability:
\begin{itemize}[leftmargin=1.25em,itemsep=1pt,topsep=1pt]
\item \textbf{Prompt-redundancy robustness.} Hold $y$ fixed across the four $x^{(j)}$'s and measure how a model's likelihood, reward, or behavior moves with prompt rephrasing.
\item \textbf{Multi-mode coverage and pass@$k$-style metrics.} Hold $x^{(j)}$ fixed and use the four responses as a built-in ``ground-truth diversity set'' against which to score generated samples.
\item \textbf{Reward-model variance estimation.} The $4{\times}4$ grid yields a clean nested-ANOVA structure for decomposing reward-model output variance into prompt-paraphrase, response-sample, and residual components.
\item \textbf{Preference-pair construction at scale.} Within-cell pairs $(y_{a}^{(j)},y_{b}^{(j)})$ are natural same-prompt comparisons; cross-cell pairs $(y_{a}^{(j)},y_{a}^{(j')})$ are paraphrase-controlled comparisons.
\item \textbf{Multi-response training research.} The structure mirrors exactly the regime our theory studies, involving multiple prompts, each with multiple responses.
\end{itemize}

\subsection{Curated Examples}

To convey the texture of the benchmark we present \emph{two} illustrative items per domain, eight items in total. For each domain, the first box exposes the \emph{prompt-paraphrase axis} (all four prompt variants of a single item, with one response shown per variant), and the second box exposes the \emph{response-sample axis} (a single prompt variant of a different item, with all four sample responses shown). All content is reproduced verbatim from the released corpus; only formatting (and, where indicated, length truncation) has been applied. Domain colors: \textcolor{mosaicCode}{\rule{1.4em}{0.55em}}\,code, \textcolor{mosaicReasoning}{\rule{1.4em}{0.55em}}\,reasoning, \textcolor{mosaicCreative}{\rule{1.4em}{0.55em}}\,creative, \textcolor{mosaicAdvisory}{\rule{1.4em}{0.55em}}\,advisory.

\clearpage
\begin{mosaicbox}[title={\texttt{code\_041}: Code Prompt Axis}]{mosaicCode}
\textbf{Base theme.} Count the number of distinct ways to climb $n$ stairs taking $1$ or $2$ steps at a time.\par\smallskip
\textbf{Four prompt variants} (each requests a \emph{different} algorithmic style).
\begin{description}[leftmargin=1.7em,style=nextline,itemsep=1pt,topsep=1pt]
\item[\texttt{P1}] Write a Python function to count how many distinct ways you can climb $n$ stairs if you can take either $1$ or $2$ steps at each move. Use a bottom-up DP table.
\item[\texttt{P2}] Implement a Python solution that computes the number of ways to reach the top of an $n$-step staircase ($1$ or $2$ steps per move) using constant extra space.
\item[\texttt{P3}] Create a recursive Python function with memoization that returns the total number of unique ways to ascend $n$ steps when each move is $1$ or $2$ stairs.
\item[\texttt{P4}] Write a Python staircase-climbing counter that uses matrix exponentiation to compute the answer in $O(\log n)$ time.
\end{description}
\textbf{One response per variant} (full code).\par\smallskip
\textit{Response to \texttt{P1}}---bottom-up DP table:
\begin{lstlisting}[style=mosaicpy]
def count_ways_dp(n):
    if n <= 0: return 0
    if n <= 2: return n
    dp = [0] * (n + 1)
    dp[1], dp[2] = 1, 2
    for i in range(3, n + 1):
        dp[i] = dp[i - 1] + dp[i - 2]
    return dp[n]
\end{lstlisting}
\textit{Response to \texttt{P2}}---constant extra space:
\begin{lstlisting}[style=mosaicpy]
def climb_constant_space(n):
    if n <= 0: return 0
    if n <= 2: return n
    a, b = 1, 2
    for _ in range(3, n + 1):
        a, b = b, a + b
    return b
\end{lstlisting}
\textit{Response to \texttt{P3}}---memoized recursion:
\begin{lstlisting}[style=mosaicpy]
from functools import lru_cache
def climb_memo(n):
    @lru_cache(maxsize=None)
    def ways(k):
        if k <= 0: return 0
        if k <= 2: return k
        return ways(k - 1) + ways(k - 2)
    return ways(n)
\end{lstlisting}
\textit{Response to \texttt{P4}}---matrix exponentiation in $O(\log n)$:
\begin{lstlisting}[style=mosaicpy]
def climb_matrix(n):
    def mat_mult(A, B):
        return [[A[0][0]*B[0][0]+A[0][1]*B[1][0], A[0][0]*B[0][1]+A[0][1]*B[1][1]],
                [A[1][0]*B[0][0]+A[1][1]*B[1][0], A[1][0]*B[0][1]+A[1][1]*B[1][1]]]
    def mat_pow(M, p):
        result = [[1, 0], [0, 1]]
        base = [row[:] for row in M]
        while p > 0:
            if p % 2 == 1: result = mat_mult(result, base)
            base = mat_mult(base, base); p //= 2
        return result
    if n <= 0: return 0
    if n <= 2: return n
    return mat_pow([[1, 1], [1, 0]], n)[0][0]
\end{lstlisting}
\textit{Each variant pins a distinct algorithmic style; the responses honor that pin while producing the same numerical answer.}
\end{mosaicbox}

\clearpage
\begin{mosaicbox}[title={\texttt{code\_115}: Code Response Axis}]{mosaicCode}
\textbf{Base theme.} Implement a function that finds the median of two sorted arrays in logarithmic time.\par\smallskip
\textbf{Prompt variant \texttt{P1}.} \textit{``Write a Python function that takes two sorted arrays and returns the median of their combined elements, achieving $O(\log(m{+}n))$ time complexity.''}\par\smallskip
\textbf{Four sample responses for \texttt{P1}}---all $O(\log\min(m,n))$, four genuinely different algorithmic structures.\par\smallskip
\textit{R1}---binary partition over the smaller array:
\begin{lstlisting}[style=mosaicpy]
def find_median_sorted_arrays(nums1, nums2):
    if len(nums1) > len(nums2): nums1, nums2 = nums2, nums1
    m, n = len(nums1), len(nums2)
    low, high = 0, m
    while low <= high:
        i = (low + high) // 2
        j = (m + n + 1) // 2 - i
        left1  = nums1[i-1] if i > 0 else float('-inf')
        right1 = nums1[i]   if i < m else float('inf')
        left2  = nums2[j-1] if j > 0 else float('-inf')
        right2 = nums2[j]   if j < n else float('inf')
        if left1 <= right2 and left2 <= right1:
            if (m + n) % 2: return float(max(left1, left2))
            return (max(left1, left2) + min(right1, right2)) / 2.0
        elif left1 > right2: high = i - 1
        else:                low  = i + 1
\end{lstlisting}
\textit{R2}---recursive $k$-th smallest with $k/2$ halving:
\begin{lstlisting}[style=mosaicpy]
def find_median_sorted_arrays(nums1, nums2):
    def kth_smallest(a, sa, b, sb, k):
        if sa >= len(a): return b[sb + k - 1]
        if sb >= len(b): return a[sa + k - 1]
        if k == 1: return min(a[sa], b[sb])
        ia = min(sa + k//2, len(a)) - 1
        ib = min(sb + k//2, len(b)) - 1
        if a[ia] <= b[ib]: return kth_smallest(a, ia + 1, b, sb, k - (ia - sa + 1))
        return kth_smallest(a, sa, b, ib + 1, k - (ib - sb + 1))
    total = len(nums1) + len(nums2)
    if total % 2: return float(kth_smallest(nums1, 0, nums2, 0, total // 2 + 1))
    return (kth_smallest(nums1, 0, nums2, 0, total // 2)
          + kth_smallest(nums1, 0, nums2, 0, total // 2 + 1)) / 2.0
\end{lstlisting}
\textit{R3}---divide-and-conquer slicing on midpoints:
\begin{lstlisting}[style=mosaicpy]
def find_median_sorted_arrays(nums1, nums2):
    def kth(a, b, k):
        if not a: return b[k]
        if not b: return a[k]
        ia, ib = len(a) // 2, len(b) // 2
        if ia + ib < k:
            if a[ia] < b[ib]: return kth(a[ia+1:], b, k - ia - 1)
            return kth(a, b[ib+1:], k - ib - 1)
        if a[ia] > b[ib]: return kth(a[:ia], b, k)
        return kth(a, b[:ib], k)
    n = len(nums1) + len(nums2)
    if n % 2: return float(kth(nums1, nums2, n // 2))
    return (kth(nums1, nums2, n//2 - 1) + kth(nums1, nums2, n // 2)) / 2.0
\end{lstlisting}
\textit{R4}---cut-based partition, single \texttt{cut1/cut2} pair:
\begin{lstlisting}[style=mosaicpy]
def find_median_sorted_arrays(nums1, nums2):
    if len(nums1) > len(nums2):
        return find_median_sorted_arrays(nums2, nums1)
    m, n = len(nums1), len(nums2)
    lo, hi = 0, m
    while lo <= hi:
        cut1 = (lo + hi) // 2
        cut2 = (m + n + 1) // 2 - cut1
        l1 = nums1[cut1-1] if cut1 else float('-inf')
        l2 = nums2[cut2-1] if cut2 else float('-inf')
        r1 = nums1[cut1]   if cut1 < m else float('inf')
        r2 = nums2[cut2]   if cut2 < n else float('inf')
        if l1 <= r2 and l2 <= r1:
            if (m + n) % 2: return float(max(l1, l2))
            return (max(l1, l2) + min(r1, r2)) / 2.0
        if l1 > r2: hi = cut1 - 1
        else:       lo = cut1 + 1
\end{lstlisting}
\textit{Four distinct algorithmic skeletons (binary-partition, recursive $k$-th, slicing divide-and-conquer, cut-based partition)---genuine algorithmic diversity, not stylistic re-wording.}
\end{mosaicbox}

\clearpage
\begin{mosaicbox}[title={\texttt{reasoning\_143}: Reasoning Prompt Axis}]{mosaicReasoning}
\textbf{Base theme.} Explain how Goodhart's Law manifests when a metric used for ranking becomes a target that agents optimize for.\par\smallskip
\textbf{Four prompt variants}---each requests a different rhetorical mode (mechanistic, advisory, illustrative, formal).
\begin{description}[leftmargin=1.7em,style=nextline,itemsep=1pt,topsep=1pt]
\item[\texttt{P1}] Explain Goodhart's Law in the context of a metric used to rank performers: how does the act of optimizing the metric degrade its informativeness?
\item[\texttt{P2}] A manager proposes ranking employees by a single quantitative KPI. Walk them through how Goodhart's Law will likely play out and what they should expect.
\item[\texttt{P3}] Using a concrete example such as standardized test scores or web search ranking, illustrate how Goodhart's Law transforms a useful measure into a target that loses correlation with the underlying goal.
\item[\texttt{P4}] Formalize Goodhart's Law: under what assumptions does optimizing a proxy metric guarantee divergence from the latent goal, and when can it be safe?
\end{description}
\textbf{One response per variant}.\par\smallskip
\textit{Response to \texttt{P1}}---mechanistic.\quad ``Goodhart's Law states that when a measure becomes a target, it ceases to be a good measure. Mechanistically, a metric is informative because it correlates with a latent quality of interest under the distribution of behaviors observed when the metric was not being optimized. Once agents are ranked by the metric, they shift their behavior toward the cheapest interventions that move the metric, which often differ from the interventions that move the underlying quality. The correlation that justified the metric was a property of the unoptimized regime, not a causal law, so optimization breaks it. Examples include teaching to the test, citation gaming, click-through manipulation, and surgeons declining risky cases to protect mortality rates. The deeper point is that metrics are compressions of multidimensional quality, and optimization exploits the compression, finding behaviors that score well without being good. Mitigations include using multiple metrics, randomizing audit dimensions, and treating metrics as diagnostics rather than incentives.''\par\smallskip
\textit{Response to \texttt{P2}}---advisory.\quad ``A single-KPI ranking is a textbook setup for Goodhart's Law. Employees will rapidly identify the cheapest behaviors that move the KPI and reallocate effort toward them, away from valuable activities the KPI does not capture. Initial gains in the headline number are often real but shallow, after which the metric continues rising while the underlying goal stagnates or declines. Side effects include neglected teamwork, hostility toward colleagues whose work indirectly helps the KPI, and creative reinterpretation of definitions. Over time the manager faces a choice: tighten definitions, in which case employees find the next loophole; or broaden the scorecard, in which case the simplicity that made the KPI attractive is lost. The honest expectation is that any high-stakes single metric will be partially gamed, that the gaming will accelerate as employees share tactics, and that the metric's signal-to-noise will degrade roughly in proportion to the stakes attached.''\par\smallskip
\textit{Response to \texttt{P3}}---illustrative.\quad ``Standardized testing is the canonical illustration. Test scores were initially designed to estimate broad academic ability and predict later achievement. Once high-stakes ranking attached to scores, schools reallocated curriculum toward tested skills, dropped untested subjects, drilled answer-elimination strategies, and in some cases manipulated who took the test. Scores rose nationally while independent measures of literacy, numeracy, and college readiness moved less or not at all. The metric was still a number on a page but it no longer measured what it had measured. The example shows the four-step pattern: a useful proxy, attachment of stakes, behavioral substitution toward the metric, and decoupling from the underlying construct. The remedy explored in education---audit testing, multiple measures, value-added models---each opens new gaming surfaces, illustrating that the law is not a single failure but a moving target.''\par\smallskip
\textit{Response to \texttt{P4}}---formal.\quad ``Let $G$ be the latent goal and $M$ the proxy metric. Under the unoptimized distribution $P_0$ of behaviors, $M = f(G) + \varepsilon$, where $f$ is monotone and $\varepsilon$ is noise correlated with $G$ in a way that depends on $P_0$. Optimization induces a new distribution $P^\star$ concentrated on behaviors that maximize $M$. If $P^\star$ lies within the support of $P_0$ and $f$ is monotone there, optimizing $M$ still optimizes $G$. The divergence guarantee arises when $P^\star$ exits the support of $P_0$, where $f$'s monotonicity is no longer guaranteed and $\varepsilon$ may be systematically large. Formally, divergence is guaranteed under bounded support of $f$'s faithful regime, sufficient optimization pressure, and the existence of high-$M$ behaviors outside the faithful regime. Safety obtains when optimization pressure is bounded below the threshold that pushes $P^\star$ off-distribution, or when $M$ is causally rather than correlationally tied to $G$ across all behaviors agents can reach.''\par\smallskip
\textit{The four responses inhabit four different registers (mechanistic explanation, managerial advice, concrete illustration, formal definition) while remaining mutually consistent on the underlying claim.}
\end{mosaicbox}

\clearpage
\begin{mosaicbox}[title={\texttt{reasoning\_204}: Reasoning Response Axis}]{mosaicReasoning}
\textbf{Base theme.} Estimate the revenue generated by all vending machines in Japan in one year.\par\smallskip
\textbf{Prompt variant \texttt{P1}.} \textit{``Estimate the annual revenue from all vending machines in Japan, explaining your assumptions.''}\par\smallskip
\textbf{Four sample responses for \texttt{P1}}---four genuinely different Fermi-estimation routes.\par\smallskip
\textit{R1}---per-machine method.\quad ``Japan has roughly $4$~million vending machines serving $125$~million people, about $1$ per $30$ people. Assume each machine sells on average $25$ items per day at an average price of $150$~yen. Daily revenue per machine $= 25 \times 150 = 3{,}750$~yen. Annual $= 3{,}750 \times 365 \approx 1.37$~M yen. Across $4$~M machines: $1.37\times10^6 \times 4\times10^6 = 5.5\times10^{12}$~yen, $\approx 5.5$~trillion yen, or roughly \$$40$~billion. Sanity check: that is about $0.9\%$ of Japan's GDP, plausible for such a ubiquitous channel.''\par\smallskip
\textit{R2}---per-capita method.\quad ``Suppose the average resident spends $30$~yen on vending purchases per day. With $125$~million people, daily national vending spend $= 3.75\times10^9$~yen. Annual $= 3.75\times10^9 \times 365 \approx 1.4\times10^{12}$~yen, $\approx 1.4$~trillion yen. Tourists and businesses raise this further, perhaps to $2$--$3$~trillion yen.''\par\smallskip
\textit{R3}---drink-volume method.\quad ``The Japanese soft-drink market sells $\approx 20$~billion liters per year, of which $\approx 30\%$ goes through vending channels: $6\times10^9$~liters. Average vended container is $0.4$~L, so $\approx 1.5\times10^{10}$ units. At $150$~yen each, drink revenue is $2.3\times10^{12}$~yen. Add hot drinks, snacks, cigarettes, and specialty items at $\approx 50\%$ more, giving $\approx 3.4\times10^{12}$~yen, or $3.4$~trillion yen.''\par\smallskip
\textit{R4}---capital-investment method.\quad ``A vending machine costs $\approx 400{,}000$~yen and is replaced every $\approx 8$ years. Operators target a $3$-year payback at $\approx 20\%$ net margin. Annual gross revenue per machine must be $\approx 400{,}000 / (3 \times 0.2) \approx 670{,}000$~yen. Across $4$~M units, total revenue $\approx 2.7$~trillion yen. Combining the four methods, my best estimate for Japan's annual vending revenue is $3$--$5$~trillion yen, roughly \$$25$--$40$~billion.''\par\smallskip
\textit{Four independent estimation routes (per-machine throughput, per-capita spending, channel volume, capital-payback) converge on $\approx 3$--$5$~trillion yen---demonstrating that response-sample diversity captures distinct epistemic strategies, not merely paraphrasings of a single calculation.}
\end{mosaicbox}

\clearpage
\begin{mosaicbox}[title={\texttt{creative\_174}: Creative Prompt Axis}]{mosaicCreative}
\textbf{Base theme.} Write a product narrative for a handmade ceramic mug brand that emphasizes imperfection as a feature.\par\smallskip
\textbf{Four prompt variants.}
\begin{description}[leftmargin=1.7em,style=nextline,itemsep=1pt,topsep=1pt]
\item[\texttt{P1}] Write a product narrative for a handmade ceramic mug brand that emphasizes imperfection as a feature.
\item[\texttt{P2}] Draft brand storytelling copy for a small ceramics studio whose handmade mugs celebrate visible imperfections as their defining quality.
\item[\texttt{P3}] Compose marketing copy for a line of artisan ceramic mugs that frames every wobble, drip, and uneven glaze as an intentional virtue rather than a defect.
\item[\texttt{P4}] Create a product narrative for a handmade mug company in which imperfection is positioned as the central selling point and emotional hook for buyers.
\end{description}
\textbf{One response per variant} (full text; brand names invented per response).\par\smallskip
\textit{Response to \texttt{P1}}---\textsc{Kiln \& Clay}.\quad ``At Kiln \& Clay, no two mugs leave the studio alike. A thumbprint pressed into wet stoneware. A glaze that pools darker on one side because the kiln runs hot near the back wall. A handle whose curve answered the maker's hand, not a machine's spec sheet. We do not sand these things away. We sign them. Every mug carries the small evidence that a person stood at a wheel, breathed, hesitated, decided. That is what you are paying for. Not a cylinder. A record. When you wrap your hands around a Kiln \& Clay mug on a Tuesday morning, the slight asymmetry meets your grip the way a good handshake does, with a little give. Perfection is a factory's apology for having no story. Imperfection is ours, told quietly, in clay.''\par\smallskip
\textit{Response to \texttt{P2}}---\textsc{Greywell Pottery}.\quad ``Greywell Pottery does not sell mugs. We sell evidence. Each piece carries the throwing rings of the wheel, the brush stroke of a cobalt slip, and the orange-peel shimmer left by salt vapor in the kiln. We could grind these traces away with sandpaper and a smoother glaze. We choose not to. They are the proof of the morning a person spent making the object you now hold every day before speaking. Customers learn to read our mugs the way a sommelier reads a label. Heavier salt deposits on the windward side. Faint crawl lines where two glazes met. A handle attachment that shows the seam. The marks are the meaning.''\par\smallskip
\textit{Response to \texttt{P3}}---\textsc{Driftwood Clay}.\quad ``Picture the mug already on your desk: smooth, symmetrical, identical to ten thousand others, and entirely forgettable. Now picture the alternative. A Driftwood Clay mug arrives with a slight lean to the handle, a crawl in the glaze near the rim, and a small unglazed patch where the maker held it during dipping. The first time you use it, the lean meets your grip in an unfamiliar way. The second time, your hand has remembered. By the second week, every other mug in your cupboard has become a stranger. We do not engineer this attachment. We simply leave the human evidence in place, where machine production scrubs it out, and let daily use complete the relationship.''\par\smallskip
\textit{Response to \texttt{P4}}---\textsc{Bramble \& Stone}.\quad ``Open any cupboard in any kitchen and you will find a graveyard of forgettable mugs: corporate giveaways, hotel souvenirs, sets that came eight to a box. None of them did anything wrong. That is the problem. Bramble \& Stone makes the mug you remove from the dishwasher and put back in front, every time. Hand-thrown in stoneware, fired in a soda kiln that paints the surface unpredictably, and finished with the throwing marks intact, our mugs do exactly one thing the cheap stack on the second shelf cannot do: they hold your attention. Imperfection is how an object earns the front of the cupboard, and our mugs earn it every wash cycle.''\par\smallskip
\textit{Each response invents its own brand identity, sensory anchor, and rhetorical posture (sign-as-signature, sommelier-of-marks, attachment-by-asymmetry, dishwasher-survivor) while staying anchored to the shared theme.}
\end{mosaicbox}

\clearpage
\begin{mosaicbox}[title={\texttt{creative\_155}: Creative Response Axis}]{mosaicCreative}
\textbf{Base theme.} Write a flash fiction piece about a baker who puts secret encouraging messages inside fortune cookies.\par\smallskip
\textbf{Prompt variant \texttt{P1}.} \textit{``Write a flash fiction piece (under $250$ words) about a baker who slips secret, hand-written encouraging messages into fortune cookies before they leave the shop.''}\par\smallskip
\textbf{Four sample responses for \texttt{P1}}---four entirely different protagonists, settings, and emotional registers.\par\smallskip
\textit{R1}---\textsc{Mei}.\quad ``Mei rolled the dough thin as moth wings before dawn. Her bakery sat between a laundromat and a closed shoe repair, and most mornings only the streetlight kept her company. She wrote the slips by hand. Not fortunes, exactly. \textit{You are doing better than you think. The hard part is almost over. Someone is proud of you, even if they have not said so.} She folded each cookie around a message, sealing it with a pinch. A delivery driver bought a dozen for his mother's birthday and called the next week to ask if she had any more. A teacher came in red-eyed on a Tuesday and left smiling. Mei kept baking. She did not believe she had any special sight, only that everyone, on some morning, needed a sentence written for them by a stranger.''\par\smallskip
\textit{R2}---\textsc{Anh}.\quad ``The bakery was small and warm. Anh started writing the slips when her daughter went to college and the house grew too quiet to bear. She wrote them at the kitchen table, in blue ink. \textit{You are allowed to rest. The work you did today mattered, even if no one saw it. Be gentler with yourself than the world will be.} She folded each cookie carefully, the paper tucked deep so the message would not fall out. Customers came in for lunch and left with paper bags of warm cookies. Sometimes one would call the bakery to ask who wrote them. Anh would say she did not know. It was easier that way. The cookies traveled across the city in coat pockets and lunch bags, small folded prayers, and Anh kept writing, one slip at a time, until her hand cramped and the morning light came in.''\par\smallskip
\textit{R3}---\textsc{David}.\quad ``David inherited the shop from his uncle. The fortune slips, when he found them, came in plastic-wrapped bricks of platitudes. He threw them out. He bought a typewriter at a yard sale and began typing his own. \textit{The thing you have been putting off is smaller than the dread of it. You have survived every difficult day so far. The person who needs to forgive you is yourself.} He cut each slip with scissors and folded them in while the cookies were still warm. Business picked up. He never knew why, exactly. People came in and stood in front of the counter for a long time, looking at the cookies as if they could see through the shells. David rang them up and wished them a good day, and they left, and the bell on the door rang behind them like a small chime of hope.''\par\smallskip
\textit{R4}---\textsc{Yuki}.\quad ``Before the bakery opened, Yuki wrote messages on slips of rice paper at her kitchen counter. She wrote one at a time, slowly, as if each one were a letter to a friend. \textit{Today is not the day you fall apart. The grief will lift, not all at once, but it will lift. You are loved, more than you know.} The cookies were folded warm and stacked in glass jars. Children pressed their faces against the glass. Old men bought single cookies and sat on the bench outside, reading the slips with their reading glasses on. Yuki watched from behind the counter and did not say anything. She knew that some of the customers would never come back, and some would come back every week, and some would forget the words entirely, and that was all right. The cookies did their work and asked nothing in return.''\par\smallskip
\textit{All four protagonists are women or men with distinct biographies, settings, motivations, and affective tones, yet each response remains a faithful fulfillment of the same prompt---the kind of within-cell diversity that response-sample studies require.}
\end{mosaicbox}

\clearpage
\begin{mosaicbox}[title={\texttt{advisory\_076}: Advisory Prompt Axis}]{mosaicAdvisory}
\textbf{Base theme.} Deciding between streaming services: how to choose the right combination for your household.\par\smallskip
\textbf{Four prompt variants}---each requests a slightly different framework lens.
\begin{description}[leftmargin=1.7em,style=nextline,itemsep=1pt,topsep=1pt]
\item[\texttt{P1}] How should a household decide which combination of streaming services to subscribe to given the overwhelming number of options?
\item[\texttt{P2}] What practical framework helps families evaluate and select the right mix of streaming platforms without overspending?
\item[\texttt{P3}] How can consumers cut through streaming service marketing to find the combination that genuinely matches their viewing habits?
\item[\texttt{P4}] What unconventional strategies help households optimize their streaming subscriptions beyond simple content comparison?
\end{description}
\textbf{One response per variant} (full text).\par\smallskip
\textit{Response to \texttt{P1}}---\textsc{audit-first.}\quad ``Audit your current viewing habits before adding or keeping any service. Check watch history on each platform to see what you actually watched in the past three months. If a service provided fewer than four hours of content your household genuinely enjoyed monthly, it is a candidate for cancellation. Many households pay for five or six services while actively using only two or three, wasting thirty to sixty dollars monthly on inertia.''\par\smallskip
\textit{Response to \texttt{P2}}---\textsc{rotation strategy.}\quad ``Use a rotation strategy instead of maintaining all subscriptions simultaneously. Subscribe to one or two services at a time, binge the content you want over one to two months, then cancel and switch to another platform. Most services allow easy cancellation and resubscription with no penalty. This approach lets a household access every major platform's best content annually while paying for only two to three services at any given time.''\par\smallskip
\textit{Response to \texttt{P3}}---\textsc{cost per actual hour.}\quad ``Calculate cost per hour of actual entertainment. If you pay fifteen dollars monthly for a service and watch twenty hours, your cost is seventy-five cents per hour, excellent value. If another service costs ten dollars but you watch only two hours monthly, that is five dollars per hour and probably not worth keeping. This metric cuts through marketing and reveals which services genuinely earn their place in your budget through actual usage rather than theoretical access.''\par\smallskip
\textit{Response to \texttt{P4}}---\textsc{paradox of choice.}\quad ``Challenge the assumption that more services equal better entertainment. Subscription fatigue is real: with too many options, households spend more time scrolling through catalogs than watching content. Psychologically, three services provide an optimal selection window. Beyond that, the paradox of choice kicks in and satisfaction per platform drops. Constraining options can paradoxically increase enjoyment by reducing decision overhead.''\par\smallskip
\textit{Each response selects a substantively different decision framework---behavioral audit, time-shifting, unit economics, choice psychology---rather than rephrasing a single recommendation list.}
\end{mosaicbox}

\clearpage
\begin{mosaicbox}[title={\texttt{advisory\_010}: Advisory Response Axis}]{mosaicAdvisory}
\textbf{Base theme.} Strategies for gradually reducing screen time before bed to improve sleep onset.\par\smallskip
\textbf{Prompt variant \texttt{P1}.} \textit{``I spend hours on my phone and laptop right up until I try to sleep, and it takes me forever to fall asleep. How can I gradually cut back on screen time before bed?''}\par\smallskip
\textbf{Four sample responses for \texttt{P1}}---four orthogonal behavioral mechanisms.\par\smallskip
\textit{R1}---\textsc{phased curfew.}\quad ``Implement a phased screen curfew that becomes stricter over three weeks. Week one, stop all screens $15$ minutes before your target bedtime. Week two, extend the buffer to $30$ minutes. Week three, aim for $60$ minutes. Use a phone alarm to signal the start of your screen-free window. During the curfew, replace screen time with a low-stimulation activity: reading a physical book, light stretching, journaling, or preparing items for the next day. The gradual ramp prevents the shock of suddenly losing your primary evening activity and allows you to discover which offline alternatives you actually enjoy before the window gets large.''\par\smallskip
\textit{R2}---\textsc{content gradient.}\quad ``Shift the content you consume rather than eliminating screens all at once. In the first phase, switch from stimulating content---social media, news, video games, work email---to passive, low-arousal content like calm nature documentaries or ambient music playlists. In the second phase, move from backlit screens to e-ink devices like a Kindle, which emit no blue light and don't suppress melatonin. In the final phase, transition to fully offline activities. This gradient approach acknowledges that screen content varies enormously in its sleep-disrupting potential, and a calm e-reader is categorically different from doomscrolling Twitter at midnight.''\par\smallskip
\textit{R3}---\textsc{environmental design.}\quad ``Use environmental design to make screens less accessible in the evening. At a set time each night, physically relocate your phone to a charging station in another room. Switch your laptop to a designated shelf. Replace the phone on your nightstand with an analog alarm clock. Remove social media apps from your phone's home screen so accessing them requires deliberate effort. The friction you add to screen access matters enormously---research on choice architecture shows that even small barriers dramatically reduce default behaviors. You're not relying on willpower; you're making the undesired behavior structurally harder to perform.''\par\smallskip
\textit{R4}---\textsc{competing ritual.}\quad ``Create a competing bedtime ritual that is genuinely more appealing than screens. This requires experimentation. Try a warm bath or shower $90$ minutes before bed---the subsequent body temperature drop triggers drowsiness. Pair it with a specific herbal tea you only drink at night, creating a sensory anchor for wind-down mode. Add five minutes of guided breathing or a body scan meditation using a non-screen audio player. The ritual needs to feel like a reward, not a punishment. If your screen-free routine feels like deprivation, you'll revert within days. Design it around pleasure, comfort, and sensory richness so it competes with screens on their own terms.''\par\smallskip
\textit{The four responses target four different behavioral-change levers (time-budget tapering, content substitution, choice-architecture friction, reward-conditioning ritual)---a within-cell sample of substantive strategy diversity.}
\end{mosaicbox}

\end{document}